\documentclass[10pt,twocolumn,letterpaper]{article}

\usepackage[pagenumbers]{cvpr} % To force page numbers, e.g. for an arXiv version

\usepackage{graphicx}
\usepackage{amsmath}
\usepackage{amssymb}
\usepackage{booktabs}

\usepackage{ntheorem}

\newtheorem{theorem}{Theorem}

\newtheorem{definition}{Definition}

\theoremheaderfont{\normalfont\itshape}
\newtheorem{claim}{Claim}
\theoremsymbol{\ensuremath{\blacksquare}} % \theoremsymbol{\ensuremath{_\square}}

\theoremstyle{nonumberplain}
\newtheorem{proof}{Proof}
 
\usepackage[hypcap=false]{caption} 

\usepackage[pagebackref,breaklinks,colorlinks]{hyperref}

 \DeclareMathOperator*{\argmax}{arg\,max}

\usepackage[capitalize]{cleveref}
\crefname{section}{Sec.}{Secs.}
\Crefname{section}{Section}{Sections}
\Crefname{table}{Table}{Tables}
\crefname{table}{Tab.}{Tabs.} 

\usepackage[inline, shortlabels]{enumitem}  % for enum with i
\usepackage{amsfonts,bm}

\usepackage{subcaption}
\usepackage{cuted}
\usepackage{capt-of}
\usepackage{dsfont}

\begin{document}

\title{Fantastic Style Channels and Where to Find Them:\\A Submodular Framework for Discovering Diverse Directions in GANs}

\author{\stepcounter{footnote}\vspace{1mm}Enis Simsar$^{1}$ 
\hspace{0.75cm}
Umut Kocasari$^{2}$ 
\hspace{0.75cm}
Ezgi Gülperi Er$^{2}$
\hspace{0.75cm}
Pinar Yanardag$^{2}$
\\
$^1$Technical University of Munich
\hspace{2em} 
$^2$Boğaziçi University
\hspace{2em} 
\\
{\tt\small enis.simsar@tum.de}
\hspace{0.5em}
{\tt\small umut.kocasari@boun.edu.tr}
\hspace{0.5em}
{\tt\small ezgi.er@boun.edu.tr}
\hspace{0.5em}
{\tt\small yanardag.pinar@gmail.com}
\vspace{-2cm}
}

\maketitle
\vspace*{-\baselineskip}
\begin{strip}\centering
\vspace{-0.3cm}
 
\centering
 \begin{minipage}{.45\textwidth}
        \centering
        \includegraphics[width=\textwidth]{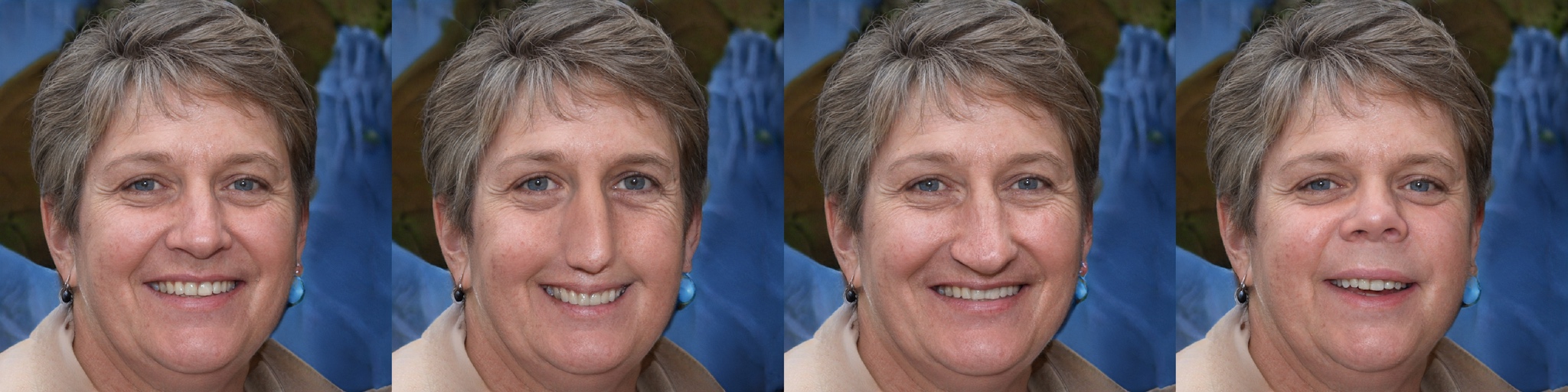}
        \vspace{-0.7cm}\captionof*{figure}{Channels from the same cluster modifying \textit{nose} in FFHQ.}
\end{minipage}
\begin{minipage}{0.45\textwidth}
        \centering
        \includegraphics[width=\textwidth]{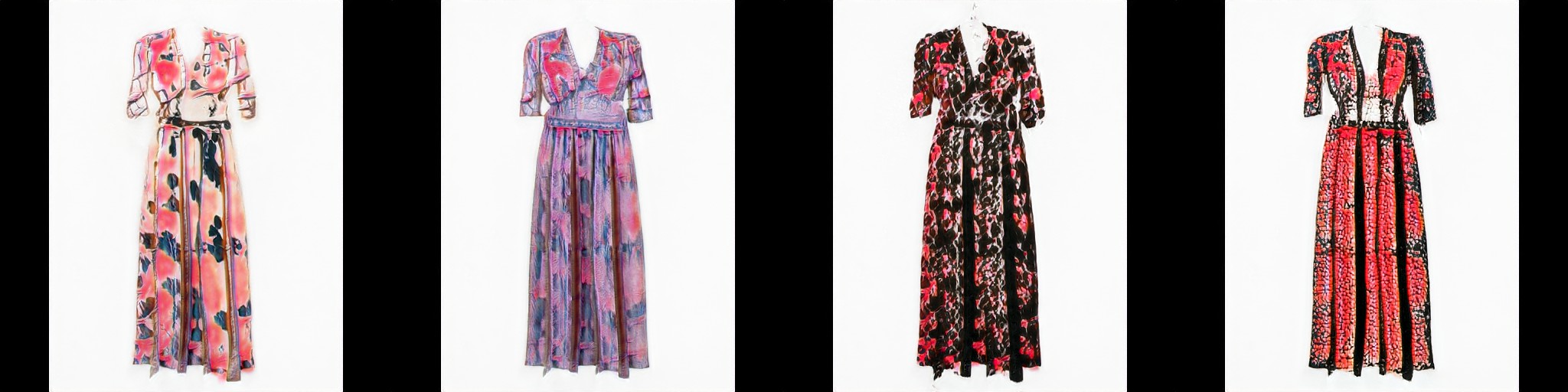}
        \vspace{-0.7cm}\captionof*{figure}{Channels from the same cluster modifying \textit{pattern} in Fashion.}
\end{minipage}
\begin{minipage}{0.9\textwidth}
       \centering
        \includegraphics[width=\textwidth]{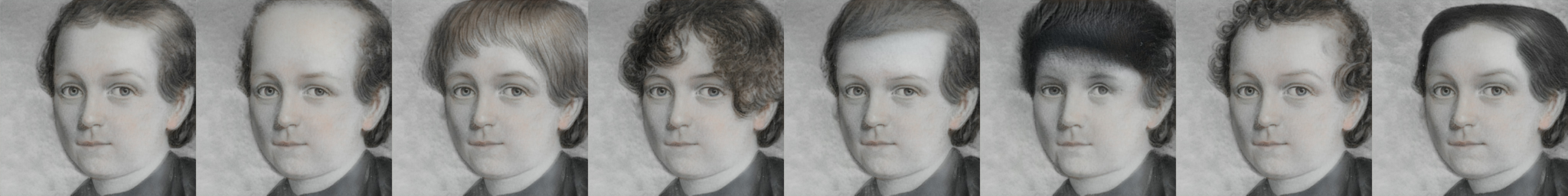}
        \vspace{-0.7cm}\captionof*{figure}{Channels from the same cluster modifying \textit{hair} in Metfaces.}
\end{minipage}
 \begin{minipage}{.45\textwidth}
        \centering
        \includegraphics[width=\textwidth]{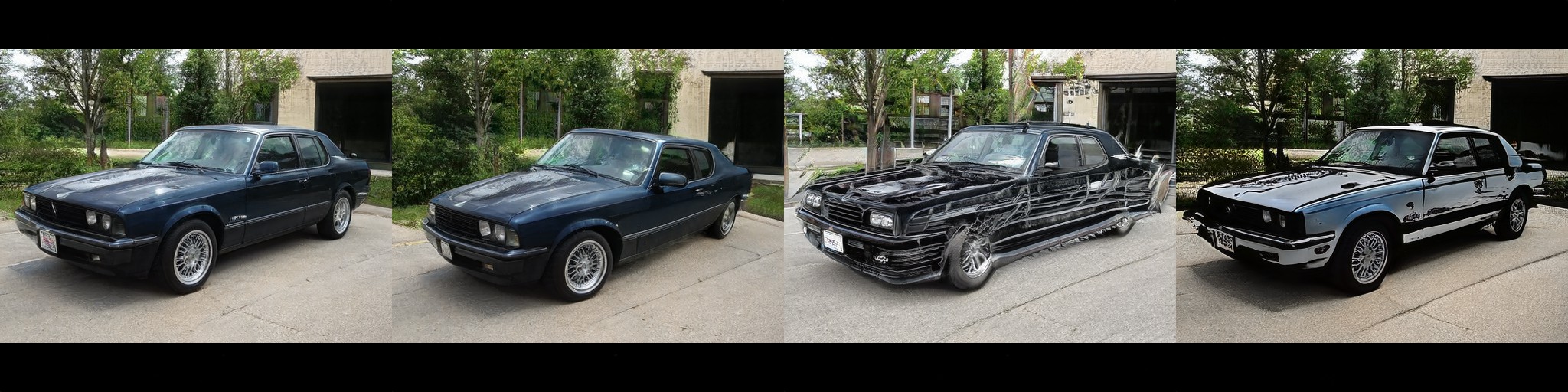}
        \vspace{-0.7cm}\captionof*{figure}{Channels from different clusters in LSUN Cars.}
\end{minipage}
\begin{minipage}{0.45\textwidth}
        \centering
        \includegraphics[width=\textwidth]{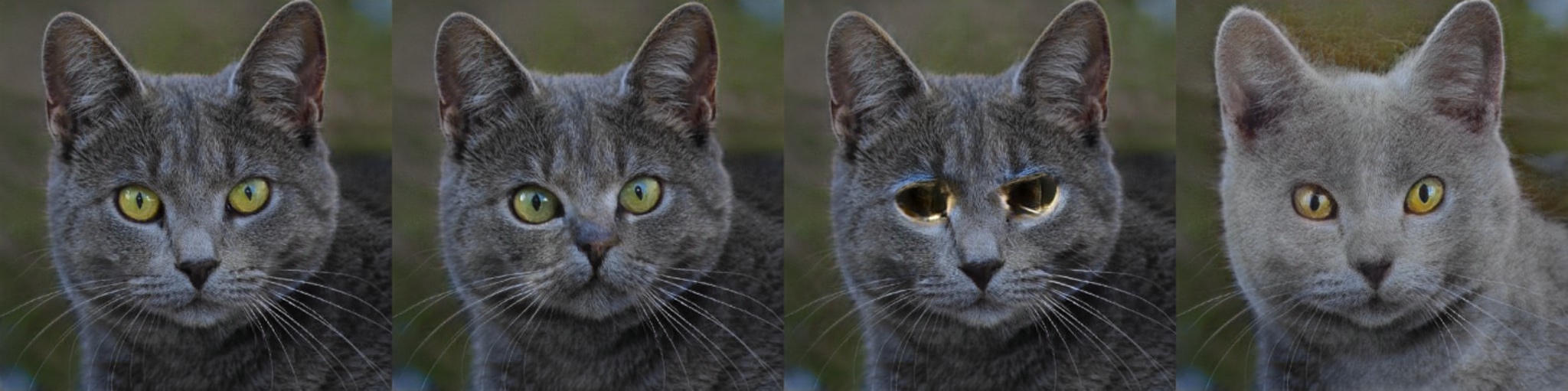}
        \vspace{-0.7cm}\captionof*{figure}{Channels from different clusters in AFHQ Cats.}
\end{minipage}
\captionof{figure}{Our submodular framework uses the notion of \textit{clusters} to select the most representative and diverse set of \textit{style channels}. Channels performing similar or different manipulations are shown in clusters above. The input images are displayed in the first column.
\label{fig:teaser}}
\end{strip}

%%%%%%%%% ABSTRACT

\begin{abstract}
The discovery of interpretable directions in the latent spaces of pre-trained GAN models has recently become a popular topic. In particular, StyleGAN2 has enabled various image generation and manipulation tasks due to its rich and disentangled latent spaces. The discovery of such directions is typically done either in a supervised manner, which requires annotated data for each desired manipulation, or in an unsupervised manner, which requires a manual effort to identify the directions. As a result, existing work typically finds only a handful of directions in which controllable edits can be made. In this study, we design a novel submodular framework that finds the most \textit{representative} and \textit{diverse} subset of directions in the latent space of StyleGAN2. Our approach takes advantage of the latent space of channel-wise style parameters, so-called stylespace, in which we cluster channels that perform similar manipulations into groups. Our framework promotes diversity by using the notion of clusters and can be efficiently solved with a greedy optimization scheme. We evaluate our framework with qualitative and quantitative experiments and show that our method finds more diverse and disentangled directions. Our project page can be found at {\small \url{http://catlab-team.github.io/fantasticstyles}}.
\end{abstract}

%%%%%%%%% BODY TEXT
\section{Introduction}
\label{sec:intro}
Recent GAN models such as StyleGAN2 \cite{StyleGAN} and BigGAN \cite{BigGAN} have achieved phenomenal success due to their ability to produce images with high visual quality and fidelity. StyleGAN, in particular, introduces a \textit{style-based} approach to transform random latent vectors into realistic images. Unlike traditional GAN architectures \cite{radford2015unsupervised,karras2017progressive}, style-based designs first transform the random latent vectors into an intermediate latent code using a mapping function, and then modify the channel-wise activation statistics of the model. Due to its rich and disentangled latent spaces, several approaches have been proposed to study the structure of the latent space of StyleGAN2 in a more principled way \cite{shen2020interfacegan,harkonen2020ganspace}. Some of these works aim to discover specific directions such as \textit{expression} or \textit{gender} using supervision \cite{shen2020interfacegan}, while others propose unsupervised approaches to identify semantically meaningful directions \cite{harkonen2020ganspace,voynov2020unsupervised}. Typically, the identified directions are used to modify the image semantics by shifting the latent code by a certain amount in the identified direction to increase or decrease the desired property.  However, while supervised methods such as \cite{shen2020interfacegan} manage to find the directions the user is interested in, they are limited since it is not always possible to find labeled data for the desired attribute. On the other hand, unsupervised methods such as \cite{harkonen2020ganspace,voynov2020unsupervised} find a certain number of directions, but the user has to manually explore what these directions are capable of. Not only this approach provide limited insight into the manipulation capabilities of latent space, but it is also time consuming for the user to explore these directions.

Recently, it has been shown that the StyleGAN2 method provides a variety of different latent spaces suitable for different image editing and manipulation tasks. For example, it has been shown that the $\mathcal{W+}$ space is suitable for image inversion \cite{abdal2019image2stylegan,tov2021designing}, while $\mathcal{S}$, the space of channel-wise style parameters (so-called \textit{stylespace}), allows disentangled edits \cite{wu2020stylespace}. This space offers rich editing capabilities where an arbitrary style channel is responsible for a particular edit, such as \textit{smile, eye color, hair type}. In other words, it is possible to perform disentangled manipulations by perturbing channel-wise style parameters of the image. While some previous work \cite{patashnik2021styleclip,wu2020stylespace} explores stylespace to find specific channels that perform a desired  in a supervised manner, what kind of manipulations stylespace has to offer in a fine-grained and unsupervised manner has not yet been explored. 

In this work we aim to find a subset of \textit{diverse} and \textit{representative} directions in latent spaces. We consider the search for directions in the latent space as a combinatorial optimization problem, where we view the latent space as a discretized set of items using the notion of style channels. Our task is then to select a  \textit{ subset} of channels that \textit{covers} the stylespace, while respecting the diversity in terms of types of manipulations they perform. This aspect is particularly important since stylespace provides more than 9K style channels and there is redundancy in what these channels cover. In particular, it has been shown that there are over 300 style channels dedicated to the control of the \textit{hair}, and over 180 channels dedicated to the \textit{ears} or \textit{background} \cite{wu2020stylespace}. Therefore, an objective function should consider \textit{diversity} into account when covering the stylespace. In other words, if a channel modifying the \textit{hair style} attribute is already selected, the gain of  covering another hair style channel should diminish. To address this issue, we design a novel framework that considers  representativeness of the channels while incorporating diversity. Our diversity objective benefits from clustering the latent space, where channels that perform similar edits are grouped under the same cluster (see Figure \ref{fig:teaser}). Our framework then ensures that selecting a channel from a cluster that has not yet been explored yields a higher gain. We formulate this task as a monotone submodular function maximization, for which there is a simple greedy algorithm guarantees that the solution obtained is almost as good as the best possible solution \cite{lin2011class}. Our contributions are as follows:

\begin{itemize} 
\item We propose the problem of finding diverse and representative style channels in the latent space of StyleGAN2 and design a submodular objective function that exhibits a natural property of diminishing returns, for which we can efficiently provide a near-optimal solution \cite{nemhauser1978analysis}. 

\item To the best of our knowledge, our framework is the first work to propose a submodular framework for finding latent directions, and the first attempt to provide a complete guide to discovering semantically meaningful \textit{groups} of style channels.

\item We share a web-based platform for navigating the stylespace at {\small  \url{http://catlab-team.github.io/styleatlas}}.
\end{itemize}
 
\section{Related Work}
\label{sec:related_work}
Recent research has shown that the latent space of GANs contains semantically meaningful directions that can be used for editing images in a variety of ways \cite{harkonen2020ganspace,jahanian2019steerability,voynov2020unsupervised}. Our approach builds on recent successes in discovering disentangled directions using \textit{stylespace}. We also benefit concepts from document summarization in the NLP literature \cite{lin2011class} to design our submodular framework. 

Several techniques are proposed to exploit the latent structure of GANs in supervised and unsupervised ways. Supervised approaches to exploit the latent space typically use pre-trained classifiers to guide the optimization process and discover directions.  \cite{shen2020interfacegan} trains a Support Vector Machine (SVM) \cite{noble2006support} with labeled data such as \textit{age, gender} and \textit{expression}. The normal vector of the resulting hyperplane is used as the latent direction. \cite{goetschalckx2019ganalyze} uses an externally trained classifier  to discover directions for cognitive image attributes in the latent space of BigGAN. Other approaches attempt to find interpretable directions in an unsupervised manner. \cite{voynov2020unsupervised} uses a classifier-based technique that finds a collection of directions that correlate with a variety of image modifications. \cite{jahanian2019steerability} presents an approach that is self-supervised and uses task-specific edit functions. \cite{shen2020closed} directly uses closed-form optimization of the intermediate weight matrix of GANs  and selecting the eigenvectors with the largest eigenvalues as directions. Ganspace \cite{harkonen2020ganspace} uses principal component analysis (PCA) \cite{wold1987principal} on randomly sampled latent vectors from the intermediate layers of BigGAN and StyleGAN2 and treats the generated principal components as latent directions. \cite{yuksel2021latentclr} uses a self-supervised contrastive learning based method to discover interpretable directions in the latent space of pre-trained BigGAN and StyleGAN2 models.

Existing work either provides limited exploration of stylespace in a supervised manner \cite{wu2020stylespace} or aims to identify relevant style channels using text-based prompts with a CLIP model \cite{patashnik2021styleclip}. In particular, \cite{wu2020stylespace} retrieves relevant channels based on a region or attribute classifier. However, what kind of manipulations stylespace has to offer in a fine-grained way has not yet been explored.

An important component of our framework requires that the channels in the stylespace are grouped into  clusters. As discussed later in Section \ref{sec:methodology}, we use the notion of \textit{clusters} to measure the diversity when covering the stylespace. There are several works that use clustering in the latent space of GAN models. We use clustering as a form of identifying similar channels and use this insight as a way to diversify coverage. \cite{collins2020editing} aims to edit an image based on a particular part of a reference image using k-means \cite{lloyd1982least}. Given a reference image and a target image, they exchange style codes based on regional differences to transfer the appearance of an object. \cite{chong2021retrieve} improves upon \cite{collins2020editing} by finding  more successful image-specific manipulation directions and eliminating the per-image matching overhead. \cite{pakhomov2021segmentation} clusters the feature maps to find meaningful and interpretable semantic classes that can be used to create segmentation masks. Compared to these methods, we aim to cluster style channels directly based on the regions they modify and use this as a way to diversify channel selection.
 
\section{Methodology}
\label{sec:methodology}

In our work, we view the latent space of StyleGAN2 as a discrete set of items using the notion of style channels in stylespace. The task we are interested in is then to select a subset of representative and diverse channels that \textit{cover} the stylespace. An overview of our framework can be found in Figure \ref{fig:framework}. Our method benefits from clustering style channels by grouping channels that perform similar manipulations. These clusters are then used in our submodular framework to promote diversity.

\subsection{Background on Stylespace}
\label{sec:stylegan}

The generation process of StyleGAN2 consists of several latent spaces, namely $\mathcal{Z}$, $\mathcal{W}$, $\mathcal{W+}$, and $\mathcal{S}$. More formally, let $\mathcal{G}$ be a generator which is a mapping function $\mathcal{G}: \mathcal{Z} \to \mathcal{X}$, where $\mathcal{X}$ is the target image domain. The latent code $\mathbf{z} \in \mathcal{Z}$ is drawn from a prior distribution $p(\mathbf{z})$, typically chosen as a Gaussian distribution. The $\mathbf{z}$ vectors are transformed into an intermediate latent space $\mathcal{W}$ using a mapper function consisting of 8 fully connected layers. The latent vectors $\mathbf{w} \in \mathcal{W}$ are then transformed into channel-wise style parameters and form the \textit{stylespace}, denoted $\mathcal{S}$, which is the latent space that determines the style parameters of the image. It has been shown that \cite{wu2020stylespace} style channels provide the most disentangled, complete, and informative space compared to others. However, it is still largely unexplored what style channels are capable of.

\begin{figure*}[ht!]
\vspace{-0.5cm}

\centering
        \includegraphics[width=0.75\textwidth]{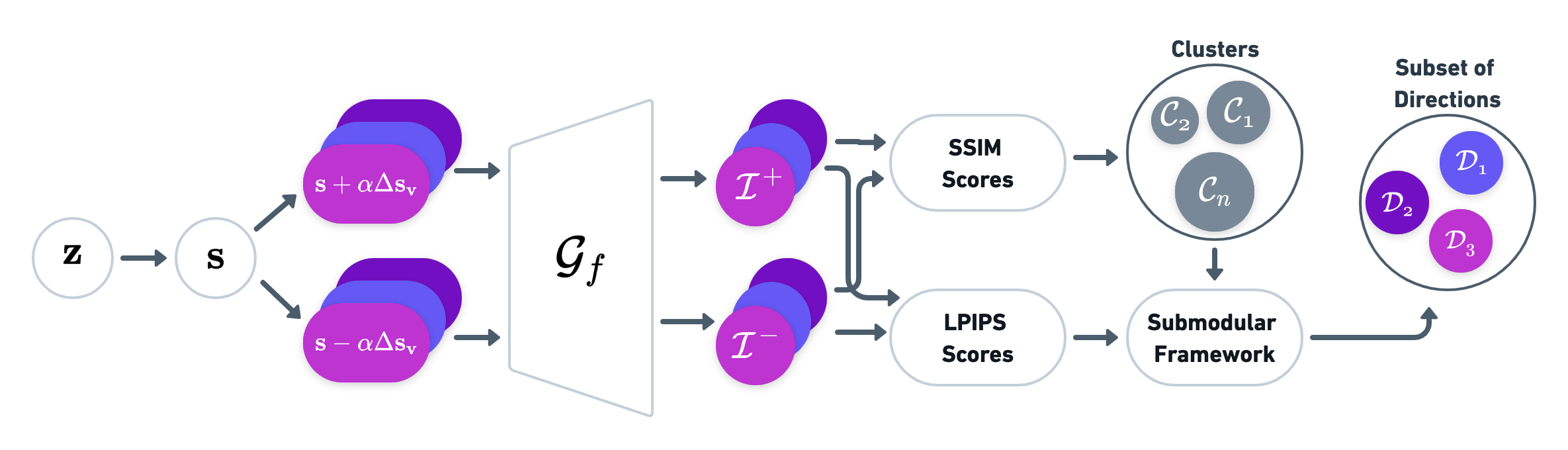}
        \vspace{-0.3cm}

 \caption{We randomly sample $M$ latent vectors $\mathbf{z} \in \mathcal{Z}$, which are transformed into style vectors $\mathbf{s}$. An arbitrary channel $v$ in $\mathcal{S}$ are perturbed by a certain amount $\alpha$ in positive and negative directions such that $(\mathbf{s}+\alpha\Delta\mathbf{s_v})$ and $(\mathbf{s}-\alpha\Delta\mathbf{s_v})$, where $\Delta\mathbf{s_v}$ is a vector containing all zeros except one of its dimensions, which is equal to one for channel $v$. LPIPS and SSIM scores are computed for the images obtained from the perturbed vectors, which are then used to generate clusters and select channels using the submodular framework. }
\label{fig:framework}
\vspace{-0.5cm}

\end{figure*}

\subsection{Background on Submodularity}
Let $\mathcal{V}$ represent a set of elements $\mathcal{V} = \{v_1, \ldots v_n\}$, often called as the \textit{ground set}. Let $\mathcal{F}: 2^\mathcal{V} \to \mathbb{R}$ represent a function that gives a real value for any subset $\mathcal{P} \subseteq \mathcal{V}$. The task we are interested in is then to select a small subset $|\mathcal{P}| \leq n$ that maximizes the function such that $\mathcal{P}^*\in\argmax_{\mathcal{P} \subseteq \mathcal{V}} \mathcal{F}(\mathcal{P})$. Solving this problem is intractable in general, but it has been shown that a greedy algorithm can be used to solve this equation almost optimally with an approximation factor of ($1-1/e$), under the condition that the function $\mathcal{F}$ is monotone, non-decreasing, and submodular \cite{sviridenko2004note}. The greedy algorithm simply starts with an empty set and at each iteration adds the item that maximizes the objective function. In other words, the solution $\mathcal{P}^*$ obtained by the greedy algorithm is a constant factor approximation to the best possible solution (say $\mathcal{P}_{\text{opt}}$) such that $\mathcal{F}(\mathcal{P}^*) \geq (1 - 1/e) \; \mathcal{P}_{\text{opt}} \approx 0.63 \; \mathcal{F}(\mathcal{P}_{\text{opt}})$. More formally, submodularity is defined as:
 
\begin{definition}
The function $\mathcal{F}$ is called submodular if for every $\mathcal{P}$ the following inequality holds: $\mathcal{F} (\mathcal{P} \cup \{ v\} ) - \mathcal{F}(\mathcal{P}) \leq \mathcal{F} (\mathcal{R} \cup \{ v \} ) - \mathcal{F}(R)$, if $R\subseteq P \subseteq \mathcal{V}$ and $v \in \mathcal{V} \setminus \mathcal{P}$.
This form of submodularity directly satisfies the diminishing returns property;
the {\em value} of the addition of $v$ never becomes larger as the context becomes larger \cite{nemhauser1978analysis}.
\end{definition}

\subsection{A Submodular Framework to Cover Stylespace}
Let $\mathcal{V}$ represent the set of style channels in the stylespace. Then, we are interested in selecting a small subset of channels $\mathcal{P} \subseteq \mathcal{V}$ that are most representative and diverse. To measure the overall \textit{coverage} or \textit{fidelity} of the channels in $\mathcal{P}$, we can define a set function as follows,

\begin{equation}
\mathcal{F}_{coverage}(\mathcal{P}) = \sum_{v_i \in \mathcal{V}, v_j \in \mathcal{P}} \mathcal{F}_{\text{sim}}(v_i, v_j)
\label{eq:coverage}
\end{equation} 
\noindent which simply computes the similarity between the summary set $\mathcal{P}$ and the ground set $\mathcal{V}$. In other words, it measures some form of coverage of $\mathcal{V}$ by $\mathcal{P}$. $\mathcal{F}_{\text{sim}}$ measures the similarity between two channels (see Section \ref{sec:cluster}). 

However, this function does not take diversity into account, since the value of the covering a particular type of edit (such as \textit{hair} or \textit{background}) never diminishes. For example, such a coverage function might favor selecting several background channels without considering diversity, since background is one of the most popular types of edits in stylespace (see Appendix \ref{app:bg}). In contrast, if we already have a channel that modifies the background in our summary set $\mathcal{P}$, then we want the gain for selecting another background channel to \textit{decrease}. A common approach is to apply a diversity regularization to our objective function \cite{lin2011class}, where we aim to reward items selected from different groups of directions such that:

\begin{equation}
\mathcal{F}_{diversity}(\mathcal{P}) = \sum_{k=1}^K  \left( \log  \left(1+ \sum_{v_i \in \mathcal{C}_k \cap \mathcal{P}}   \mathcal{F}_{\text{reward}}({v_i}) \right) \right) 
\label{eq:diversity}
\end{equation}

\iffalse
\begin{equation}
\mathcal{F}_{diversity}(\mathcal{P}) = \log \left( 1 + \sum_{k=1}^K  \left( \sum_{v_i \in \mathcal{C}_k \cap \mathcal{P}}  \mathcal{F}_{\text{reward}}({v_i}) \right)  \right) 
\label{eq:diversity}
\end{equation}
\fi
\noindent where the ground set $\mathcal{V}$ of style channels is partitioned into $K$ separate clusters. The clusters $\mathcal{C}_k$ are disjoint, where  $k=1, \ldots K$ and  $\bigcup_k \mathcal{C}_k = \mathcal{V}$. For each style channel $v_i$, we have a reward $\mathcal{F}_{\text{reward}}({v_i}) \geq 0$, which indicates the importance of adding channel $v_i$ to the empty set (see Section \ref{sec:reward}). 

Let us explain the intuition behind $\mathcal{F}_{diversity}$ in more detail. The idea is that when a channel is selected, the gain decreases for channels from the same cluster due to the concave function $\log(1+x)$. For example, suppose that the candidate channels in cluster $\mathcal{C}_1$ are $v_1$ and $v_2$, which have rewards of $5$ and $4$, respectively. Similarly, the cluster $\mathcal{C}_2$ has a candidate channel, $v_3$ with a score of $3$. When we evaluate the objective function in Eq. (\ref{eq:diversity}) for the first time, we select $v_1$ since it has the largest marginal gain. However, the next time we choose channel $v_3$, even though the score of $v_2$ is higher, because 
 $\log(5 + 4) < \log(5) + \log(3)$. Intuitively, this means that selecting a channel from a cluster that has not yet been explored will yield a higher gain than selecting a channel from a cluster that we  already covered. Thus, the objective function rewards diversity by selecting elements from different clusters and prevents popular channels such as \textit{background} from dominating the selected set.

Then, the overall objective function we want to solve is a combination of both:

 \begin{equation}
 \mathcal{F}(\mathcal{P}) =   \mathcal{F}_{coverage}(\mathcal{P}) + \lambda \mathcal{F}_{diversity}(\mathcal{P})  
 \label{eq:submod_channels}
 \end{equation}
\noindent where $\lambda \geq 0$ is the tradeoff coefficient between coverage and diversity. Since we are interested in selecting a small subset, we aim to maximize the following objective function,

\begin{equation}
\mathcal{P}^* = \argmax_{\mathcal{P} \subseteq \mathcal{V}: |\mathcal{P}| \leq n}  \mathcal{F}(\mathcal{P})
\label{eq:argmax}
\end{equation} 
\noindent subject to a cardinality constraint $n$, which denotes the total number of channels in the set $\mathcal{P}^*$. This objective function combines two aspects in which we are interested: 1) it encourages the selected set to be \textit{representative} of the stylespace, and 2) it positively rewards \textit{diversity}. Finding the exact subset that maximizes this equation is intractable. However, it has been shown that maximizing a monotone submodular function under a cardinality constraint can be solved near optimally using a {\em greedy} algorithm \cite{nemhauser1978analysis}. In particular, if a function $\mathcal{F}$ is submodular, monotone and takes only non-negative values, then a greedy algorithm  approximates the optimal solution of the Eq. \eqref{eq:argmax} within a factor of $(1 - 1/e) $ \cite{nemhauser1978analysis}. Note that this property is particularly attractive because it is a worst-case bound. In most cases, the quality of the obtained solution of submodular optimization problems is much better than this bound suggests \cite{lin2011class}.

\begin{theorem}
Given two functions $\mathcal{F}: 2^\mathcal{V} \to \mathbb{R}$ and $f: \mathbb{R} \to \mathbb{R}$, the composition $\mathcal{F}' = f \, \circ \, \mathcal{F}: 2^\mathcal{V} \to \mathbb{R}$  is non-decreasing submodular, if $\mathcal{F}$ is non-decreasing concave and $\mathcal{F}$ is non-decreasing
submodular. \cite{lin2011class}
\label{theorem:concave}
\end{theorem}

\begin{claim}: The function in Eq. \eqref{eq:submod_channels} is submodular.
\end{claim}
\begin{proof} The $\mathcal{F}_{coverage}(\mathcal{P})$ is a sum of modular functions  with non-negative weights (hence, monotone). Similarly, the sum of non-negative rewards in $\mathcal{F}_{diversity}(\mathcal{P})$ is also monotone. This monotone function is surrounded by a non-decreasing concave function $\log(1+x)$. Applying a concave function to a monotone function, we obtain a submodular function (see Theorem \ref{theorem:concave}). Finally, the sum of a collection of submodular functions is submodular \cite{stobbe2010efficient}, so $\mathcal{F}(\mathcal{P})$ in Eq. \eqref{eq:submod_channels} is submodular.
\end{proof}

\subsubsection{Reward of channels} \label{reward}
\label{sec:reward}
 
Our framework requires a singleton reward associated with each style channel for the diversity objective. To this end, we use the LPIPS\cite{zhang2018unreasonable} metric as a proxy for the reward score, where channels with more perceptual changes have a larger value. First, we sample $M$ random latent vectors $\mathbf{z}$ $\in$ $\mathcal{Z}$ and pass them through the mapping network of StyleGAN2 to obtain their corresponding style vectors $\mathbf{s}$. Given an arbitrary channel $v\footnote{We drop the subscript of $v$ in the rest of this paper for clarity.} \in \mathcal{S}$, we perturb the value of channel $v$ in each style vector $\mathbf{s}$, while leaving the other channels unchanged, and generate modified images, $\mathcal{G}(\mathbf{s}+\alpha\Delta\mathbf{s_v})$ and $\mathcal{G}(\mathbf{s}-\alpha\Delta\mathbf{s_v})$. $\Delta\mathbf{s_v}$ is a vector containing all zeros except one of its dimensions, which is equal to one for channel $v$, and $\alpha$ denotes the magnitude of the perturbation. We run both images through the VGG16 \cite{simonyan2014very} network and compute the L2 difference between their feature embeddings. This process is repeated for $M$ style vectors and the average LPIPS score for each channel $v$ is calculated as the reward value $\mathcal{F}_{\text{reward}}(v)$.

\subsubsection{Clustering Stylespace} 
\label{sec:cluster} 
Our method quantifies diversity by using the notion of clusters $\mathcal{C}_k, k=1, \ldots K$, where channels performing similar edits are grouped together. Using the same approach as above, we first obtain the perturbed images for each style vector $\mathbf{s}$ such that $\mathcal{G}(\mathbf{s}+\alpha\Delta\mathbf{s_v})$ and $\mathcal{G}(\mathbf{s}-\alpha\Delta\mathbf{s_v})$. Then we compute the structural similarity index (SSIM) \cite{wang2004image}, which is a metric for measuring the similarity between two  images. In particular, we obtain the image difference between two images, where the difference is represented as a value in the range [0, 255]. This process is repeated for each style channel in $\mathcal{S}$ for a total of $M$ style vectors, resulting in $|\mathcal{S}| \times M$ matrices of SSIM scores. We then compute the cosine distance between the SSIM matrix of each style channel, with the distance between channels averaged over $M$ style vectors. We use the resulting matrix as a distance matrix in agglomerative clustering \cite{florek1951liaison} to cluster style channels into groups. We have experimented with both agglomerative clustering and k-means algorithms and found that they yield similar clusters. We use agglomerative clustering to group the channels since it uses a precomputed distance matrix to speed up the clustering process and does not require tuning the number of clusters. We note that clustering at individual layers leads to finer-grained  clusters for models such as FFHQ. For such large models, we perform clustering at each layer and then group the clusters based on the regions they modify (e.g., \textit{hair, ear, background}) using a segmentation model \cite{lee2020maskgan}.

SSIM scores are also used to compute the similarity between two style channels. Given two style channels $v_i$ and $v_j$, $\mathcal{F}_{\text{sim}}$ in Eq. (\ref{eq:coverage}) is calculated as the cosine similarity between the SSIM matrix of each channel, averaged over $M$ style codes.

%\footnote{ \url{https://github.com/scikit-image/scikit-image/blob/4839ae40e5c013fbbb05cae09060199a2470931b/skimage/measure/_structural_similarity.py\#L14}}

\section{Experiments}
\label{sec:experiments}
We conduct several qualitative experiments to demonstrate the effectiveness of the submodular framework and compare our method to supervised \cite{wu2020stylespace} and unsupervised methods \cite{harkonen2020ganspace,shen2020closed}. We also explore clusters of StyleGAN2 on a variety of datasets, including FFHQ \cite{StyleGAN}, LSUN Cars \cite{yu2015lsun}, AFHQ Cats \cite{Choi2020StarGANVD}, Metfaces \cite{Karras2020TrainingGA}, and Fashion. For Fashion model, we train a StyleGAN2 model with dataset collected from \cite{farfetch,netaporter}. Finally, we present two applications that leverage our framework to allow users to explore stylespace.

\begin{figure}[t!]\centering

\phantom{a}
\begin{minipage}{.47\columnwidth}
    \centering
    \scriptsize{\textbf{Cluster 1}}
    \vspace{0.05cm}
\end{minipage}
% \begin{minipage}{0.3\columnwidth}
%      \centering 
%     \textbf{Cluster 2}
%      \vspace{0.05cm}
% \end{minipage}
\begin{minipage}{.47\columnwidth}
    \centering 
    \scriptsize{\textbf{Cluster 2}}
    \vspace{0.05cm}
\end{minipage}

\rotatebox[origin=lc]{90}{\centering \hspace{-0.35cm} \scriptsize{\textbf{FFHQ}}}
\begin{minipage}{.47\columnwidth}
         \centering
 \includegraphics[width=\textwidth]{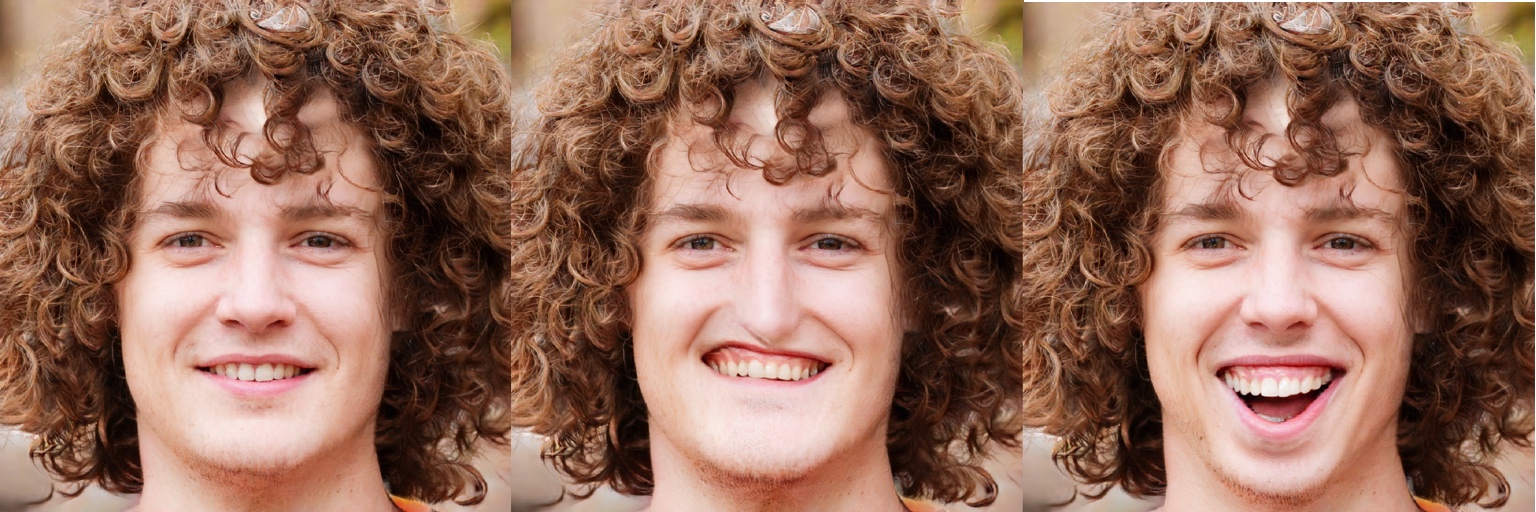}
         \vspace{-1.02cm}\captionof*{figure}{ }
 \end{minipage}
%  \begin{minipage}{0.3\columnwidth}
%          \centering
%           \includegraphics[width=\textwidth]{images/three/6_220_3_6_220_6_8_339_5_seed_14_power_60.jpeg}
%          \vspace{-1.02cm}\captionof*{figure}{}
%  \end{minipage}
\begin{minipage}{.47\columnwidth}
         \centering
          \includegraphics[width=\textwidth]{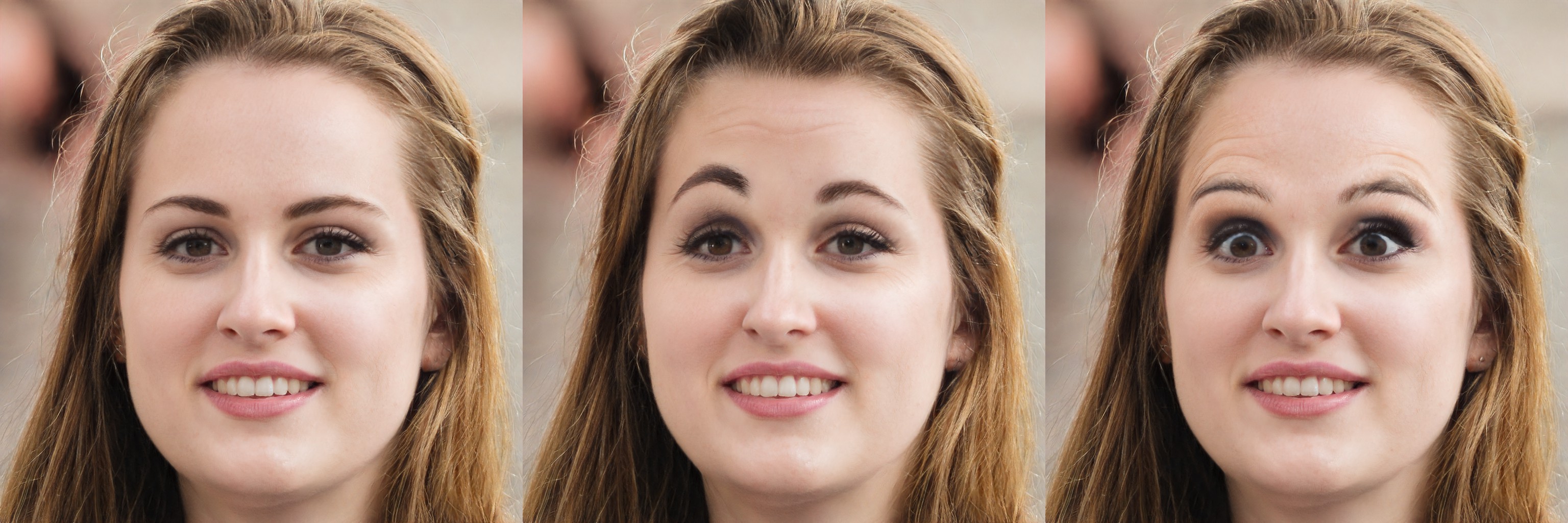}
         \vspace{-1.02cm}\captionof*{figure}{}
 \end{minipage}

 \rotatebox[origin=lc]{90}{\centering \hspace{-0.4cm} \scriptsize{\textbf{Fashion}}}
\begin{minipage}{.47\columnwidth}
     \centering
     \includegraphics[width=\textwidth]{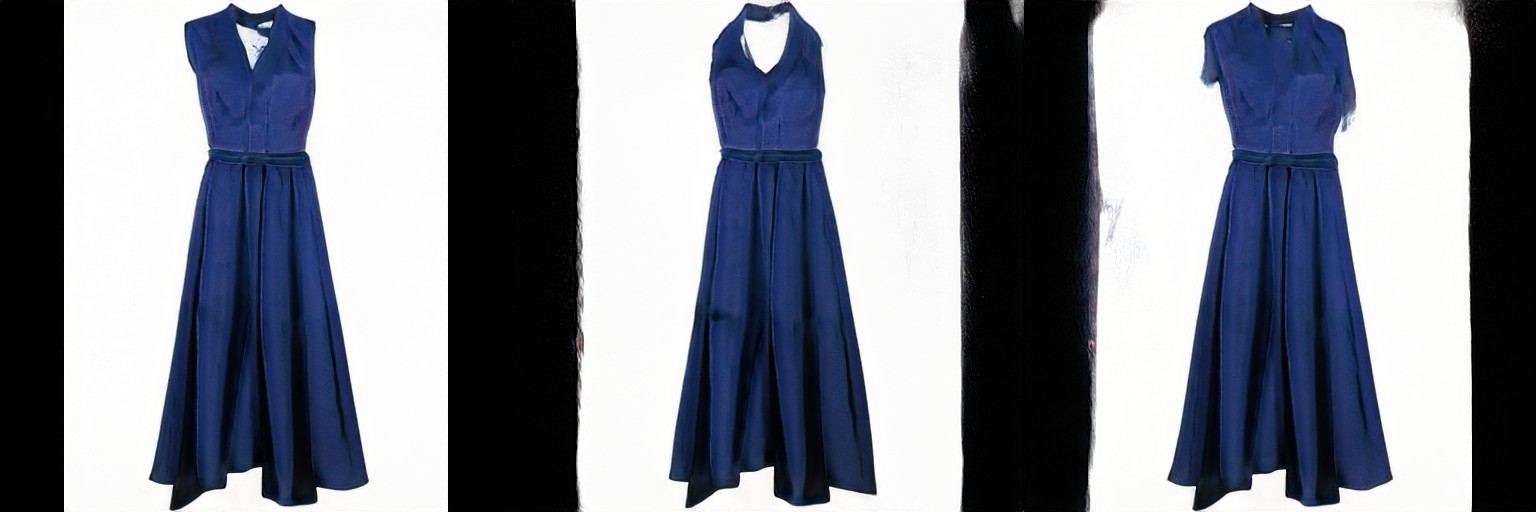}
     \vspace{-1.02cm}\captionof*{figure}{ }
\end{minipage}
% \begin{minipage}{0.3\columnwidth}
%      \centering
%      \includegraphics[width=\textwidth]{images/finals/12_19_3_12_19_2_12_369_1_seed_13_power_150.jpeg}
%      \vspace{-1.02cm}\captionof*{figure}{}
% \end{minipage}
\begin{minipage}{0.47\columnwidth}
     \centering
     \includegraphics[width=\textwidth]{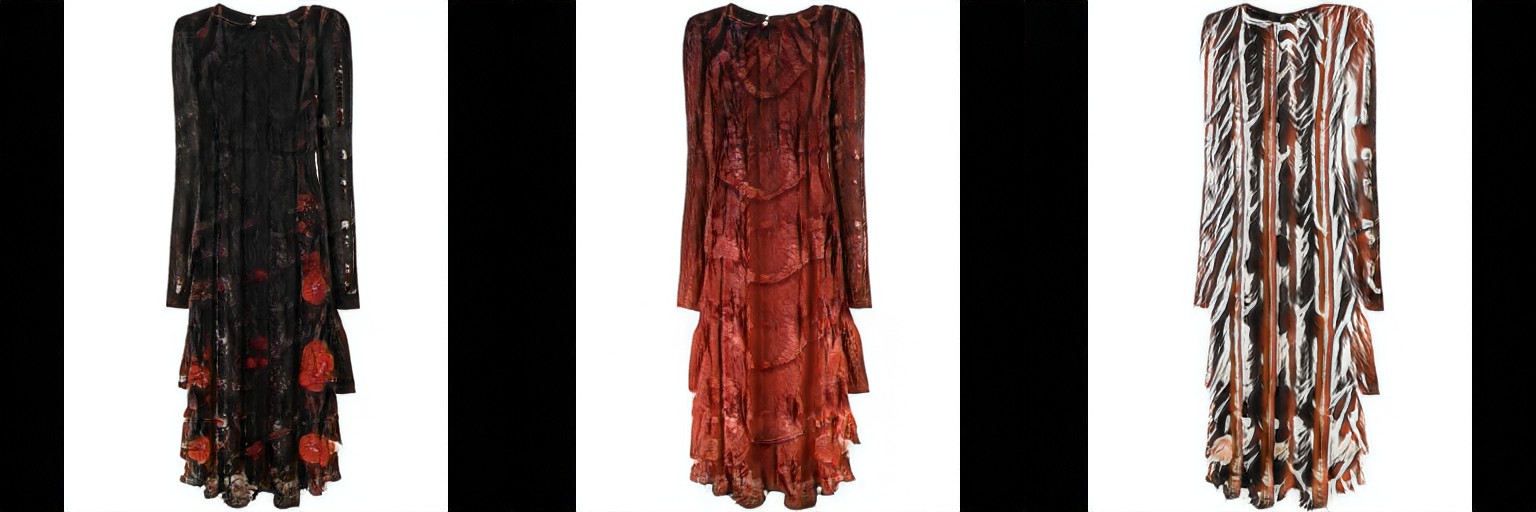}
     \vspace{-1.02cm}\captionof*{figure}{}
\end{minipage}

\rotatebox[origin=lc]{90}{\centering \hspace{-0.65cm} \scriptsize{\textbf{AFHQ Cats}}}
\begin{minipage}{.47\columnwidth}
     \centering
        \includegraphics[width=\textwidth]{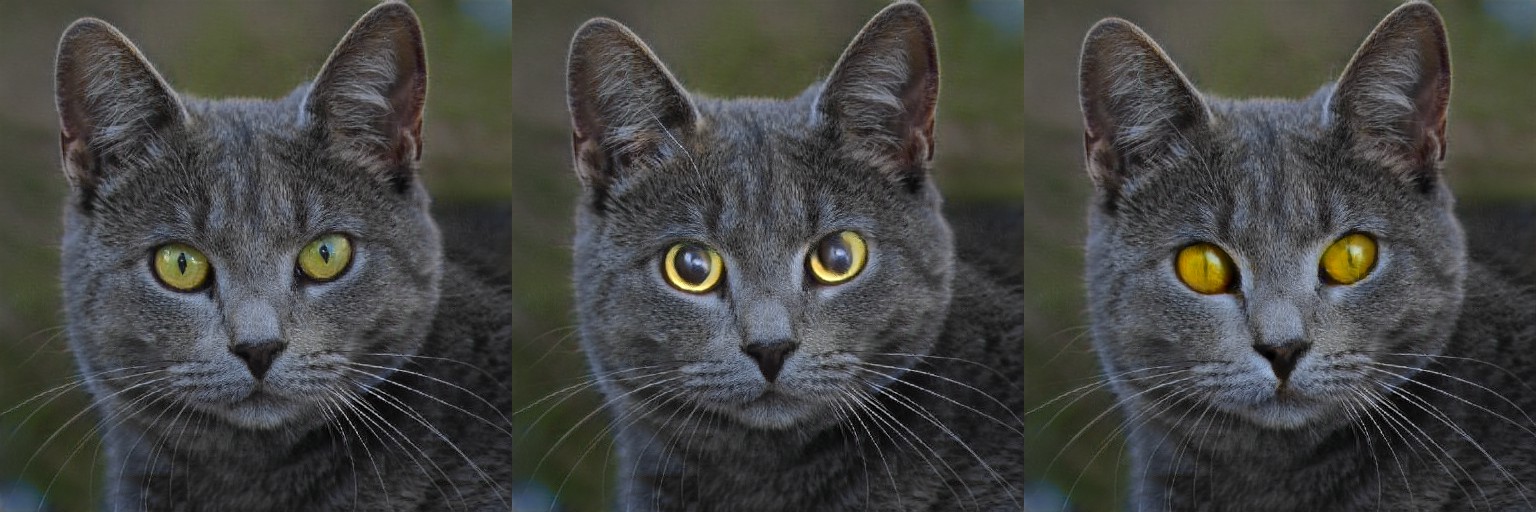}
     \vspace{-1.02cm}\captionof*{figure}{ }
\end{minipage}
% \begin{minipage}{0.3\columnwidth}
%      \centering
%      \includegraphics[width=\textwidth]{images/finals/6_473_3_9_228_2_9_261_2_seed_6_power_100.jpeg}
%      \vspace{-1.02cm}\captionof*{figure}{}
% \end{minipage}
\begin{minipage}{.47\columnwidth}
     \centering
     \includegraphics[width=\textwidth]{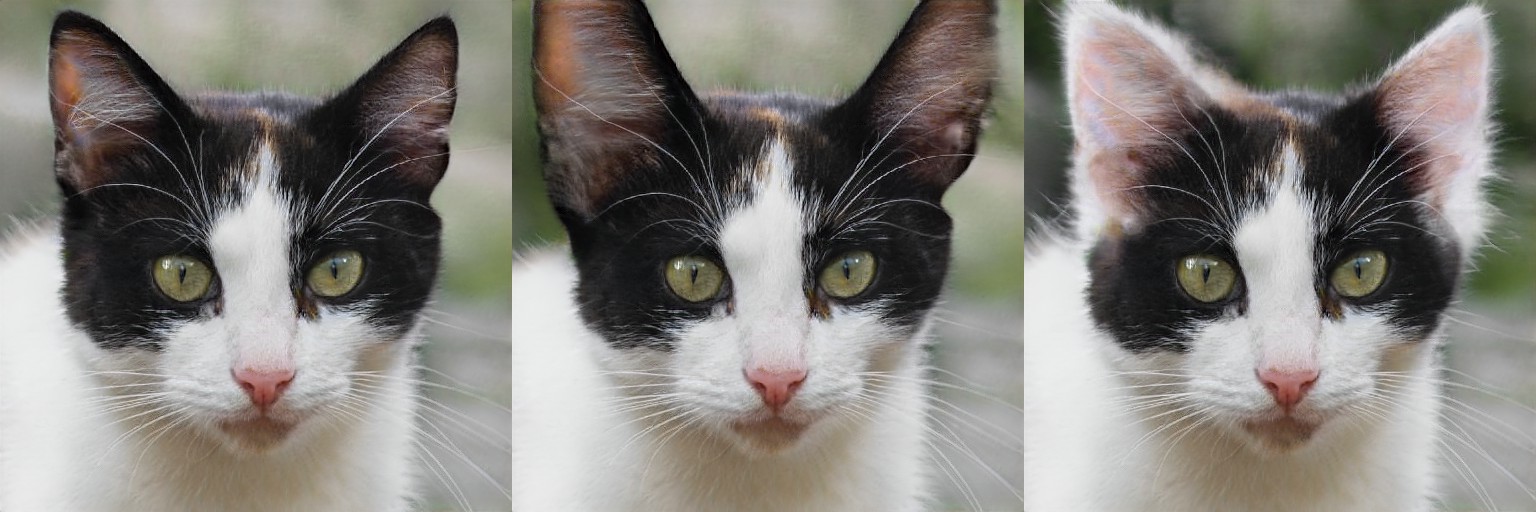}
     \vspace{-1.02cm}\captionof*{figure}{}
\end{minipage}

 \rotatebox[origin=lc]{90}{\centering \hspace{-0.65cm} \scriptsize{\textbf{LSUN Cars}}}
\begin{minipage}{.47\columnwidth}
         \centering
         \includegraphics[width=\textwidth]{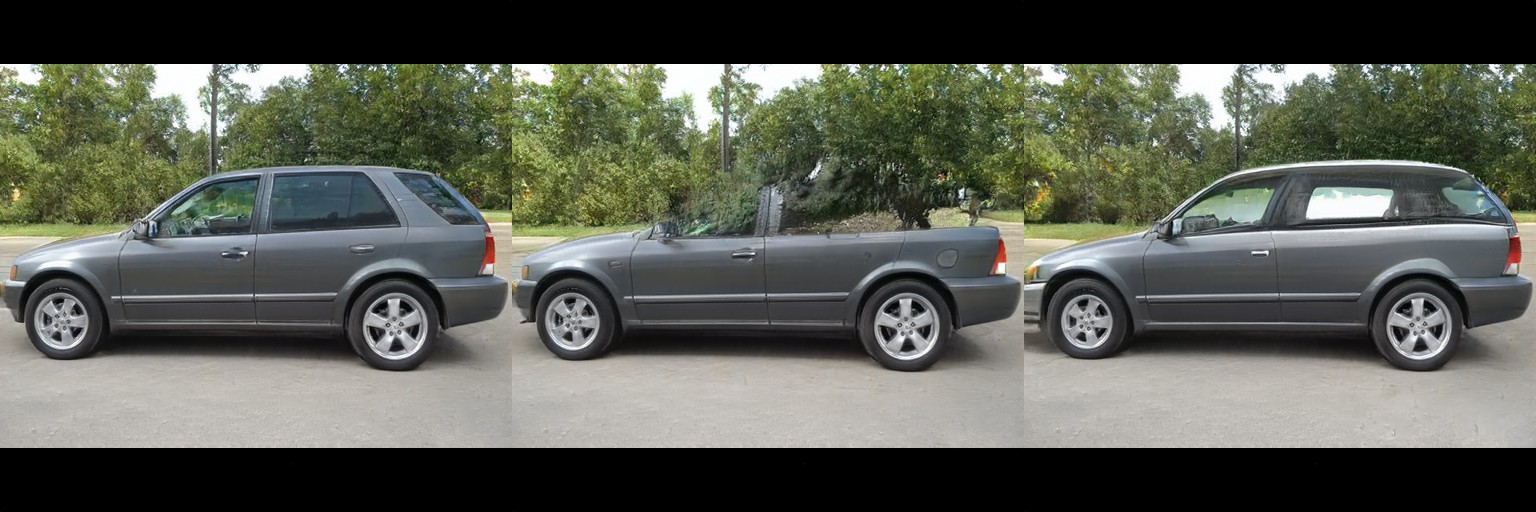}
         \vspace{-1.02cm}\captionof*{figure}{ }
 \end{minipage}
%  \begin{minipage}{0.3\columnwidth}
%          \centering
%          \includegraphics[width=\textwidth]{images/finals/9_159_3_9_426_2_12_309_2_seed_75_power_200.jpeg}
%          \vspace{-1.02cm}\captionof*{figure}{}
%  \end{minipage}
\begin{minipage}{.47\columnwidth}
         \centering
         \includegraphics[width=\textwidth]{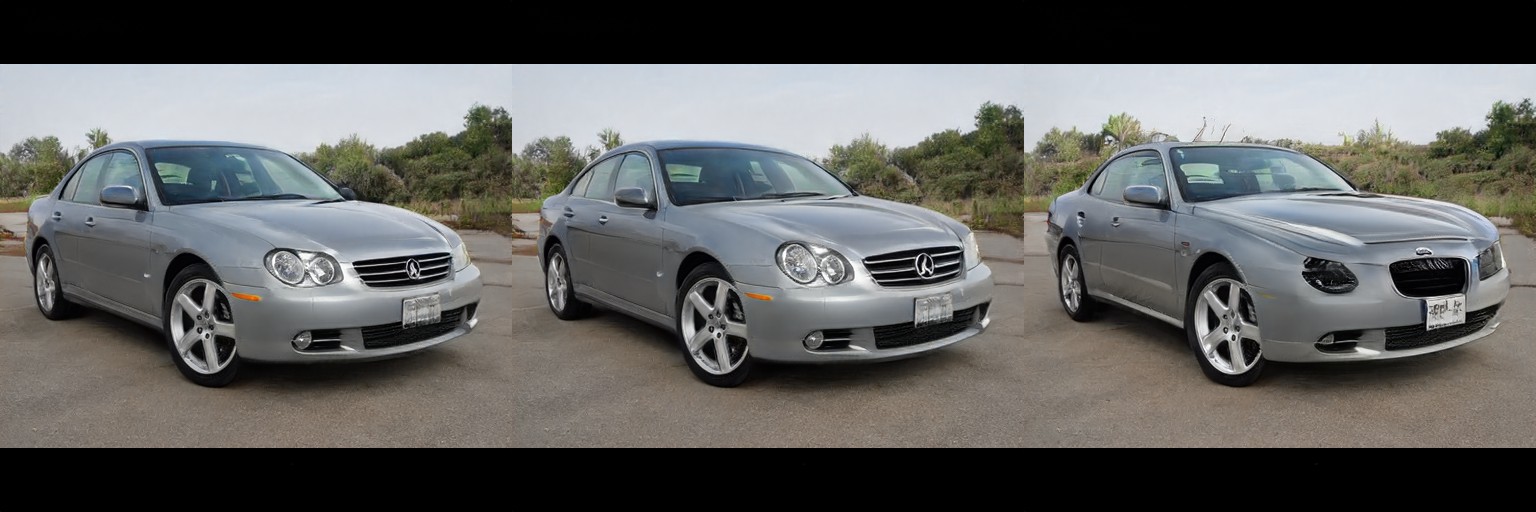}
         \vspace{-1.02cm}\captionof*{figure}{}
 \end{minipage}
 
 \rotatebox[origin=lc]{90}{\centering \hspace{-0.5cm} \scriptsize{\textbf{Metfaces}}}
\begin{minipage}{.47\columnwidth}
         \centering
         \includegraphics[width=\textwidth]{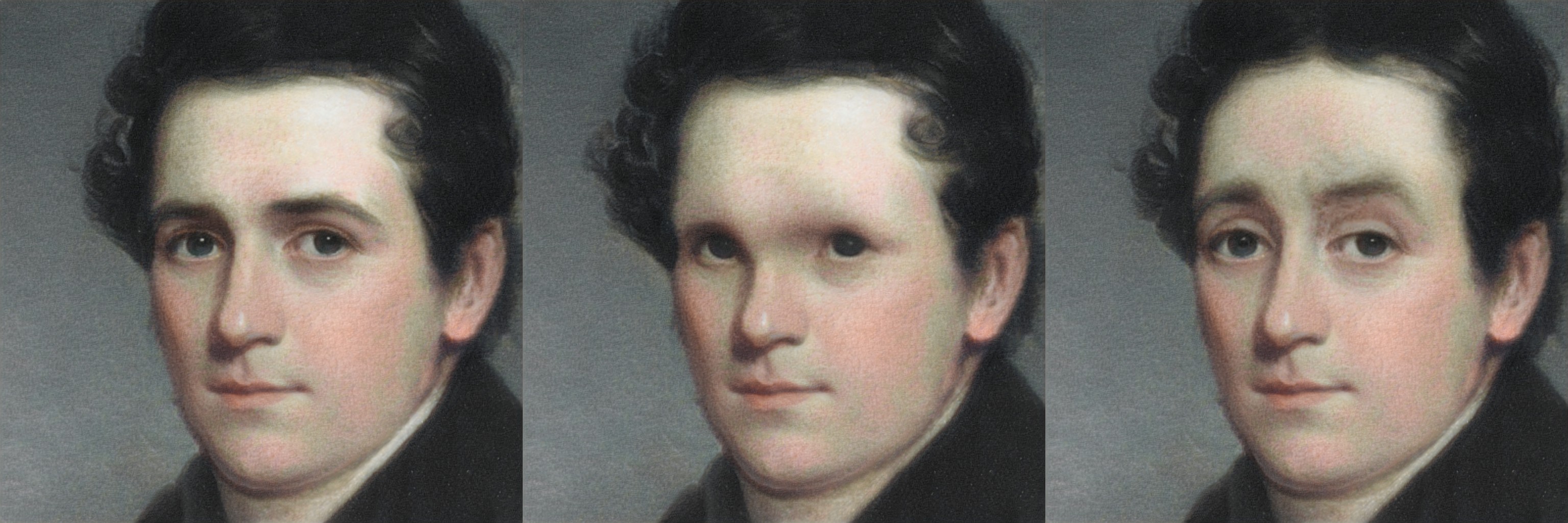}
         \vspace{-1.02cm}\captionof*{figure}{ }
 \end{minipage}
%  \begin{minipage}{0.3\columnwidth}
%          \centering
%          \includegraphics[width=\textwidth]{images/met/6_433_3_6_96_1_6_465_6_seed_77_power_130.jpeg}
%          \vspace{-1.02cm}\captionof*{figure}{}
%  \end{minipage}
\begin{minipage}{.47\columnwidth}
         \centering
         \includegraphics[width=\textwidth]{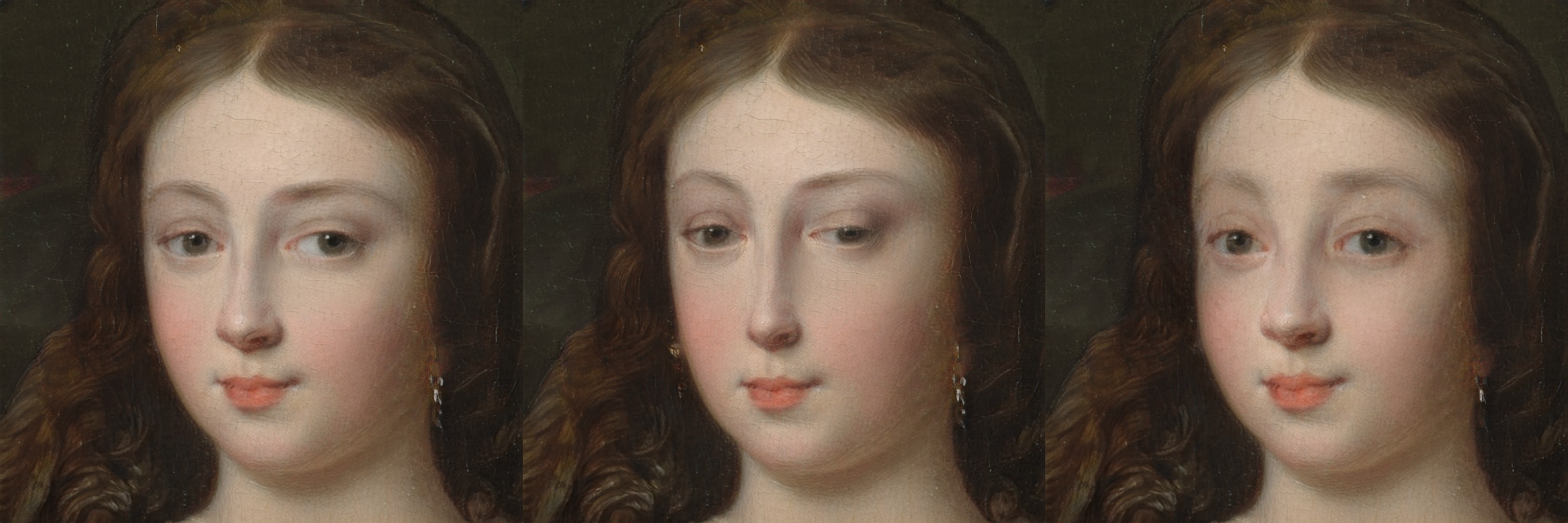}
         \vspace{-1.02cm}\captionof*{figure}{}
 \end{minipage}

 \vspace{-0.2cm}

\captionof{figure}{\label{fig:dress-cat-met-car} \textbf{Sample clusters for various datasets}. The first column represents the input image and the remaining columns show the manipulation performed by a random channel in the cluster.}
\vspace{-0.5cm}
\end{figure}

\begin{figure}[t!]\centering

\rotatebox[origin=lc]{90}{\centering \hspace{-0.6cm} \textbf{Layer 6}}
\begin{minipage}{.7\columnwidth}
     \centering
     \includegraphics[width=\textwidth]{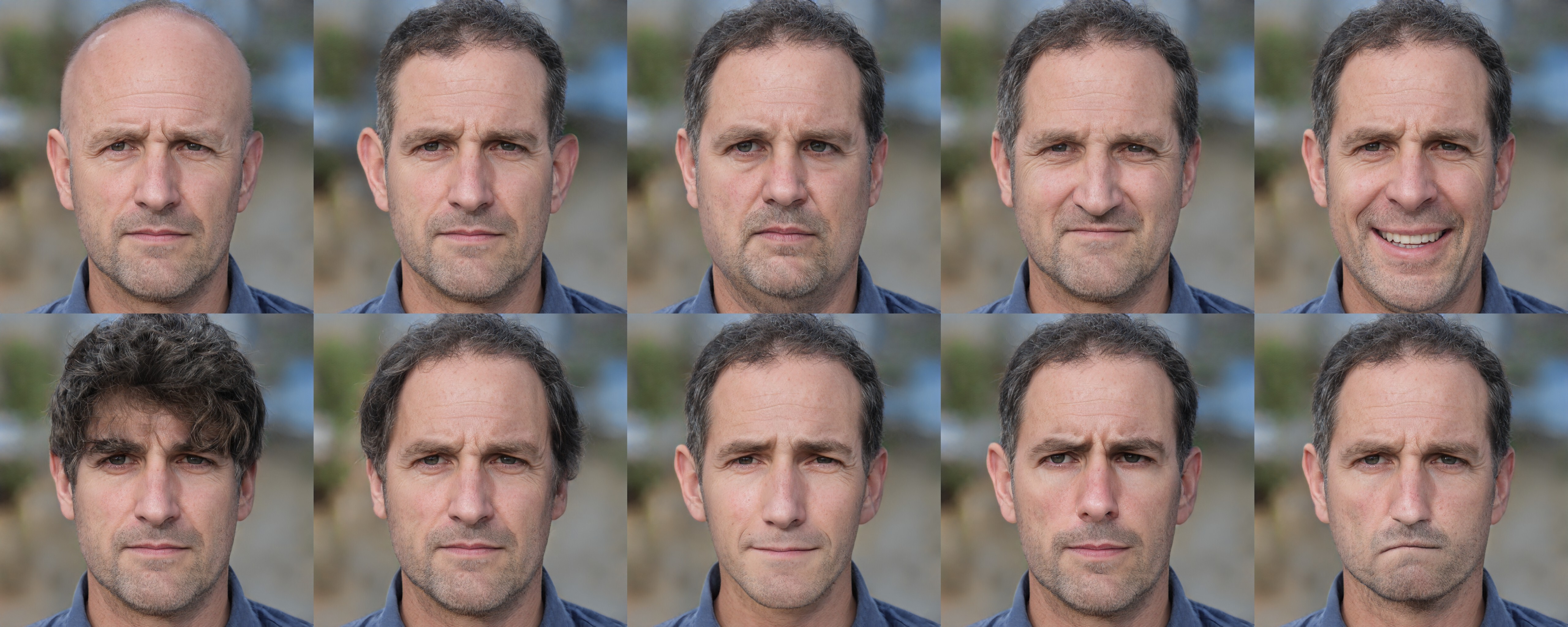}
     \vspace{-1.02cm}\captionof*{figure}{ }
\end{minipage}
\rotatebox[origin=lc]{90}{\centering \hspace{-0.8cm} \textbf{$\alpha+$} \hspace{0.4cm}  \textbf{$\alpha-$}}

\rotatebox[origin=lc]{90}{\centering \hspace{-0.6cm} \textbf{Layer 8}}
\begin{minipage}{.7\columnwidth}
     \centering
     \includegraphics[width=\textwidth]{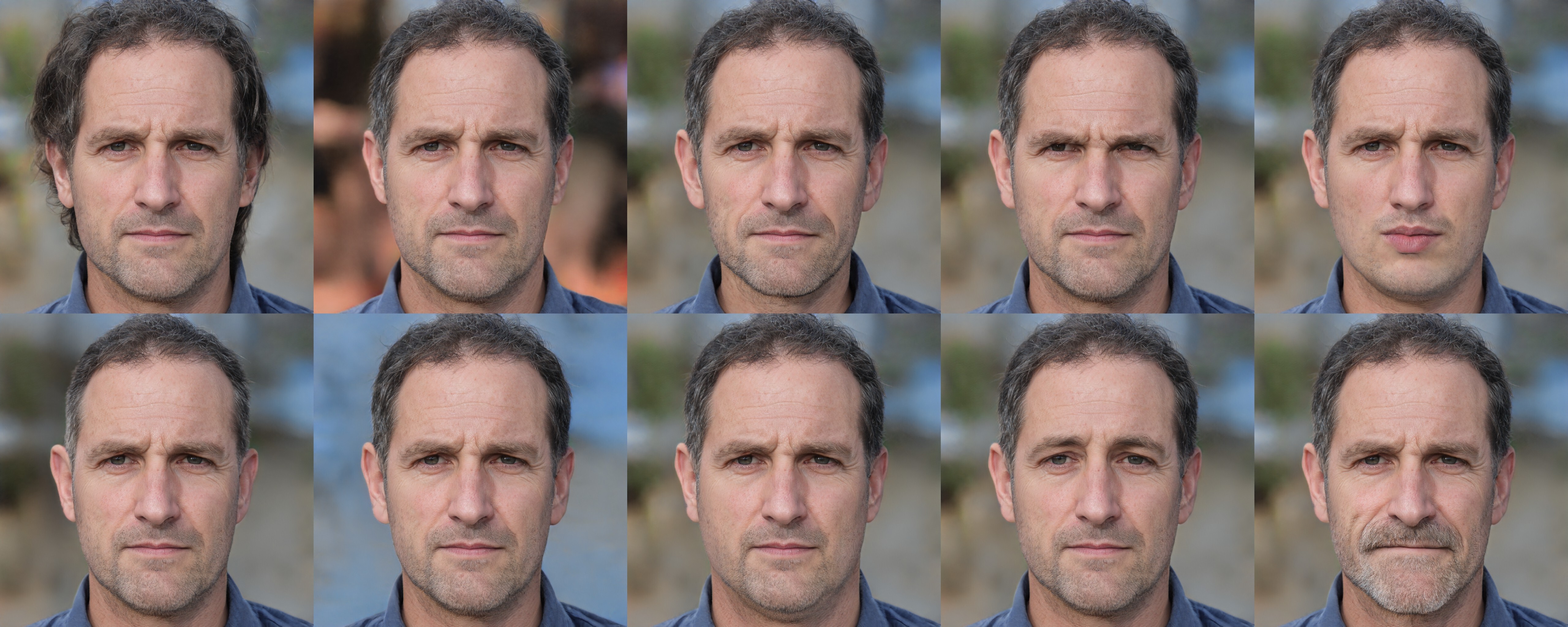}
     \vspace{-1.02cm}\captionof*{figure}{}
\end{minipage}
\rotatebox[origin=lc]{90}{\centering \hspace{-0.8cm} \textbf{$\alpha+$} \hspace{0.4cm}  \textbf{$\alpha-$}}

\rotatebox[origin=lc]{90}{\centering \hspace{-0.6cm} \textbf{Layer 9}}
\begin{minipage}{.7\columnwidth}
     \centering
     \includegraphics[width=\textwidth]{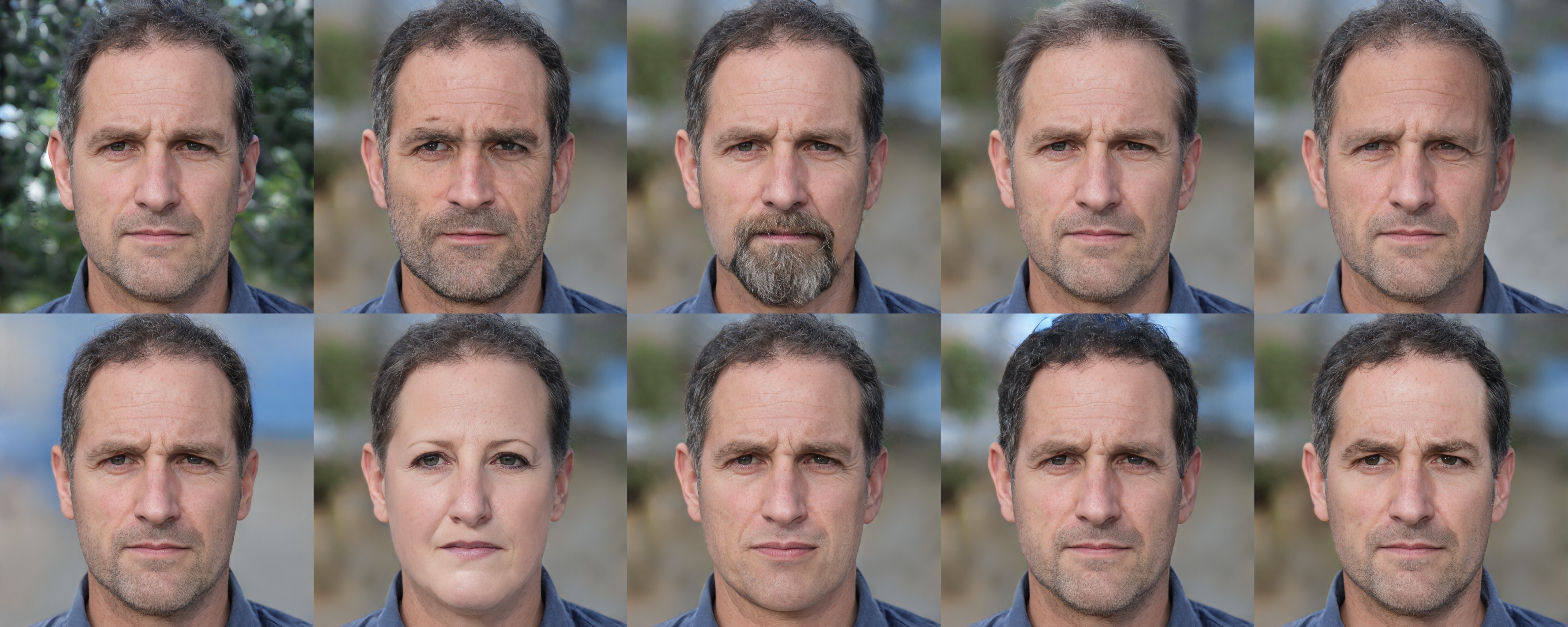}
     \vspace{-1.02cm}\captionof*{figure}{ }
\end{minipage}
\rotatebox[origin=lc]{90}{\centering \hspace{-0.8cm} \textbf{$\alpha+$} \hspace{0.4cm}  \textbf{$\alpha-$}}

\rotatebox[origin=lc]{90}{\centering \hspace{-0.6cm} \textbf{Layer 12}}
\begin{minipage}{.7\columnwidth}
     \centering
     \includegraphics[width=\textwidth]{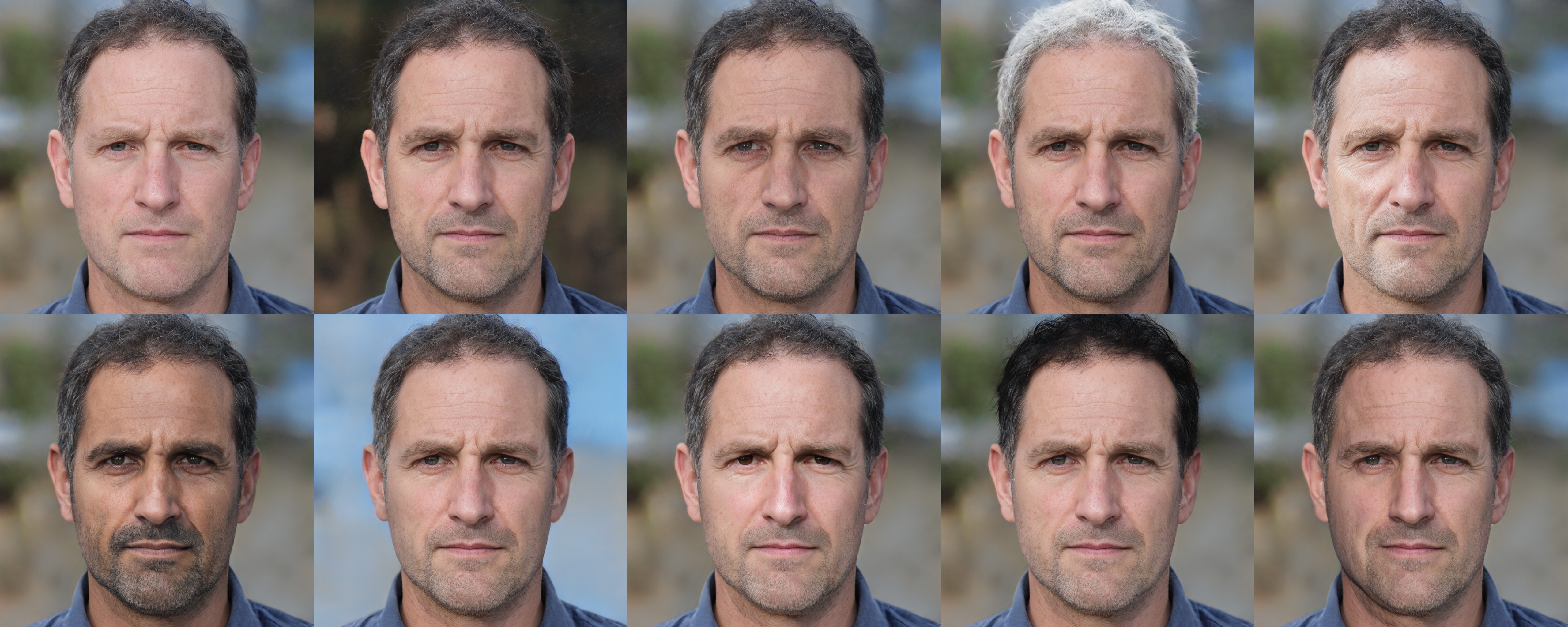}
     \vspace{-1.02cm}\captionof*{figure}{}
\end{minipage}
\rotatebox[origin=lc]{90}{\centering \hspace{-0.8cm} \textbf{$\alpha+$} \hspace{0.4cm}  \textbf{$\alpha-$}}

  \vspace{-0.2cm}

\captionof{figure}{\label{fig:layerwise-submodular} \textbf{Top 5 channels ranked by our submodular framework for individual layers}. As can be seen from the results, our method is able to select diverse channels for each layer.}
 \vspace{-0.5cm}

\end{figure}

\subsection{Experimental Setup}
For all experiments, we use the StyleGAN2 model \cite{Karras2020TrainingGA} with truncation value \textit{0.7}. For the LPIPS and SSIM scores, $\alpha$ is set to 20 and the number of style codes is set to $M=128$. It takes 1 hour to compute  LPIPS and SSIM scores. We use Scikit-learn \cite{scikit-learn} for agglomerative clustering, with the distance threshold parameter set to \textit{0.7}, resulting in about 20 to 40 clusters depending on the layer. Clustering per layer takes 5-15 seconds. Following \cite{wu2020stylespace}, we exclude RGB layers as they cause entangled manipulations, and we exclude the last 4 blocks as they represent very fine-grained features that are difficult to use for editing tasks. For the submodular framework, we use the diversity tradeoff $\lambda$ as $25$. For our experiments, we use a single NVIDIA Titan RTX GPU. 

\subsection{Qualitative Results}
 \textbf{Clustering Stylespace} Our submodular framework relies on the clusters to encourage  diversity. Figure \ref{fig:teaser} and Figure \ref{fig:dress-cat-met-car} show  clusters from the FFHQ, Fashion, AFHQ Cats, LSUN Cars, and Metfaces datasets. We note that clusters that modify similar regions are grouped together, such as \textit{smile, expression} in FFHQ, \textit{neck type, pattern} in Fashion, \textit{eye color, ear type} in AFHQ Cats, \textit{roof type, bumper type} in LSUN Cars, \textit{eyebrow type, expression} in Metfaces, shown in Figure \ref{fig:dress-cat-met-car}.

\textbf{Covering Stylespace} Our framework is flexible in terms of which groups of layers to cover. We can choose to cover only channels from a single layer or from multiple layers. In either case, one just needs to form the clusters based on the particular layers of interest. Next, we investigate both cases.

\begin{itemize}

\item \textbf{Single layers} We first experiment with selecting a subset of channels on single layers. Figure \ref{fig:layerwise-submodular} shows the top 5 channels for individual layers $L=6,8,9,12$. We see that our framework selects diverse channels for each layer, such as channels that modify \textit{hair, ear, face, expression, mouth} as in layer $L=6$ or \textit{background, gender, beard, hair, expression} as in layer $L=9$. Note that performing submodular ranking allows us to get the top channels for each layer, but is still not sufficient to cover the stylespace, as channels that perform similar edits may be ranked at top for different layers and cause redundancy. For example, channels that change \textit{background} are placed at the top in different layers (see first and second channels in layers $L=9, 8, 12$, respectively). Therefore, submodular selection at multiple layers is required to achieve adequate stylespace coverage, as we show below. 

\item \textbf{Multiple layers} Figure \ref{fig:comparison} shows the top 10 channels ranked by our method considering multiple layers. As can be seen from the results, our method selects a variety of channels that modify regions such as \textit{background, hair, face, mouth, eye, ear, and clothing}. We note that our method places a channel that modifies \textit{background} first, as this is one of the most popular types of editing offered by the stylespace. Covering another \textit{background} channel then has diminishing returns thanks to the submodularity property, and preference is given to channels that modify other diverse regions before placing another background channel at the $8^{th}$ position. 
\begin{figure}[t] 
\centering
        \includegraphics[width=0.9\columnwidth]{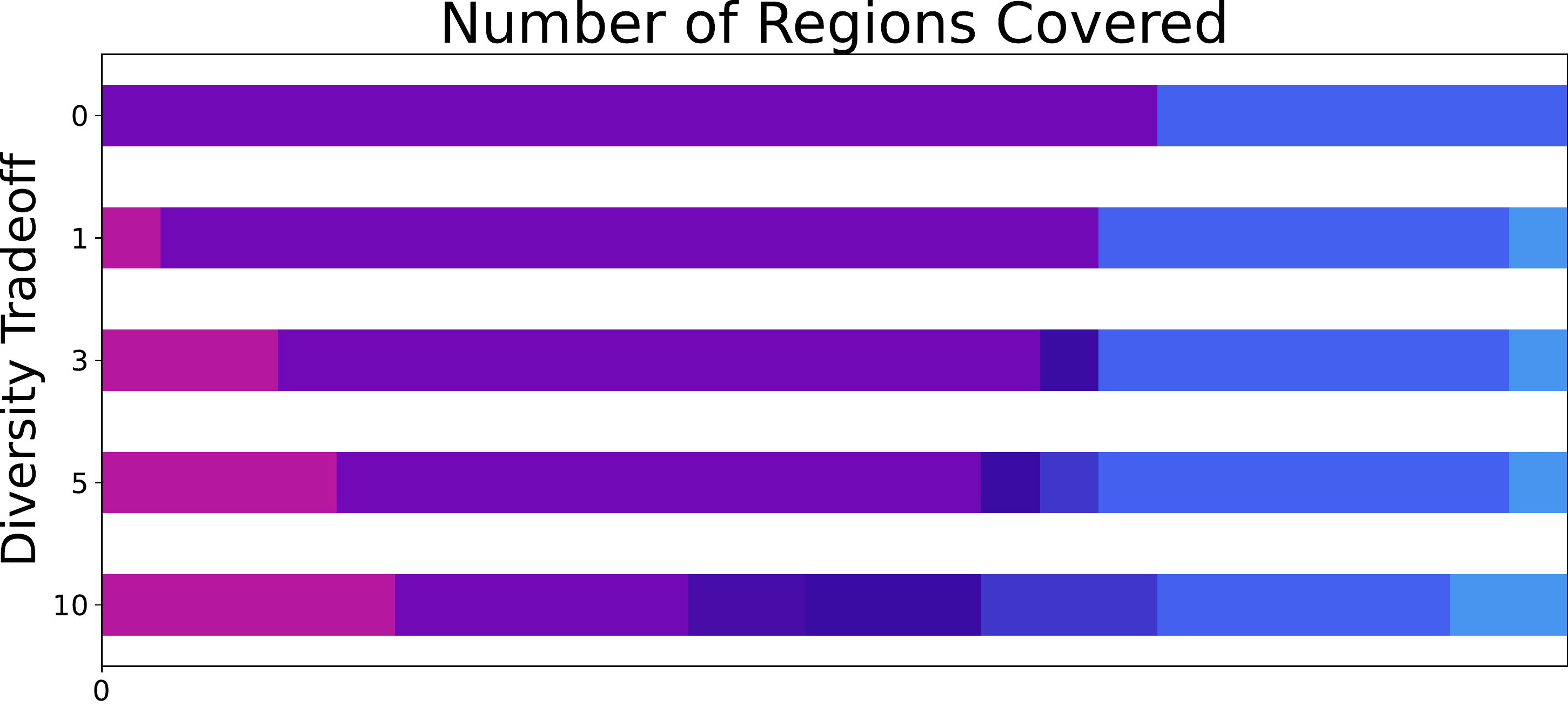}
        %\vspace{-0.5cm}

\caption{\textbf{The effect of the diversity tradeoff}. The number of regions (indicated by different colors) covered by the top 25 channels in FFHQ. Our model covers more regions as we increase the diversity tradeoff $\lambda$ due to diminishing returns.}
\vspace{-0.5cm}
\label{fig:diversity}
\end{figure}

\item \textbf{Diversity tradeoff} We also examine the effects of the diversity parameter $\lambda$ (see Figure \ref{fig:diversity}). When the diversity parameter $\lambda = 0$, we find that the number of regions in the top 25 channels covers only two regions. When we increase the parameter $\lambda$, we find that more regions are covered and the balance between regions improves since the submodular framework accounts for diversity.
\end{itemize}
 
\begin{figure*}[t!]
\centering
\rotatebox[origin=lc]{90}{\centering \hspace{-0.6cm} \textbf{Ours\phantom{p}}}
\begin{minipage}{0.075\textwidth}
     \centering
     \includegraphics[width=\textwidth]{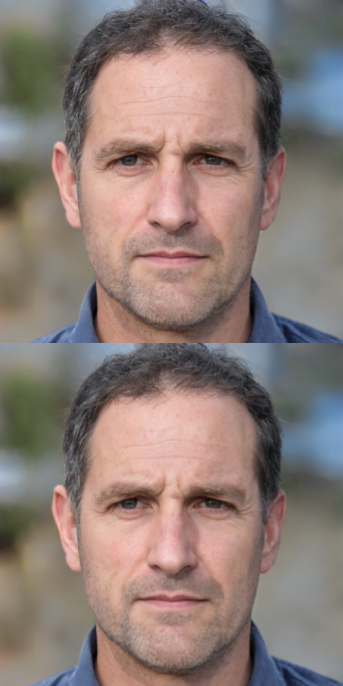}
\end{minipage}
% \hspace{0.18cm}
\begin{minipage}{0.75\textwidth}
     \centering
     \includegraphics[width=\textwidth]{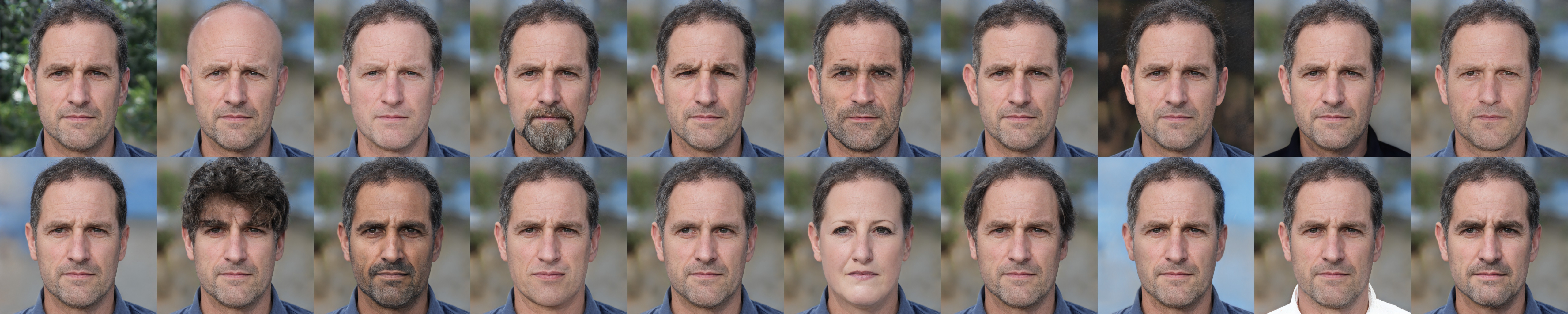}

 %    \vspace{-1.02cm}\captionof*{figure}{}
\end{minipage}
\rotatebox[origin=lc]{90}{\centering \hspace{-1cm} \textbf{$\alpha+$} \hspace{0.7cm}  \textbf{$\alpha-$}}

\rotatebox[origin=lc]{90}{\centering \hspace{-1.02cm}   \textbf{Ganspace} }
\begin{minipage}{0.075\textwidth}
     \centering
     \includegraphics[width=\textwidth]{images/seed_1_grid.png}
\end{minipage}
% \hspace{-0.18cm}
\begin{minipage}{0.75\textwidth}
\centering
     \includegraphics[width=\textwidth]{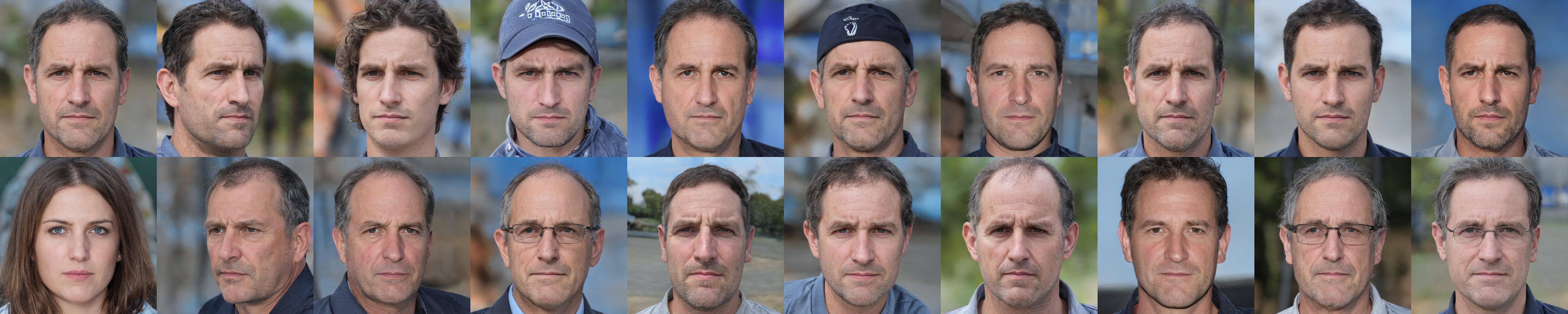}
   %  \vspace{-1.02cm}\captionof*{figure}{ }
\end{minipage}
\rotatebox[origin=lc]{90}{\centering \hspace{-1cm} \textbf{$\alpha+$} \hspace{0.7cm}  \textbf{$\alpha-$}}

\rotatebox[origin=lc]{90}{\centering \hspace{-0.6cm} \textbf{SeFa\phantom{p}}}
\begin{minipage}{0.075\textwidth}
     \centering
     \includegraphics[width=\textwidth]{images/seed_1_grid.png}
\end{minipage}
% \hspace{-0.18cm}
\begin{minipage}{0.75\textwidth}
    \centering
     \includegraphics[width=\textwidth]{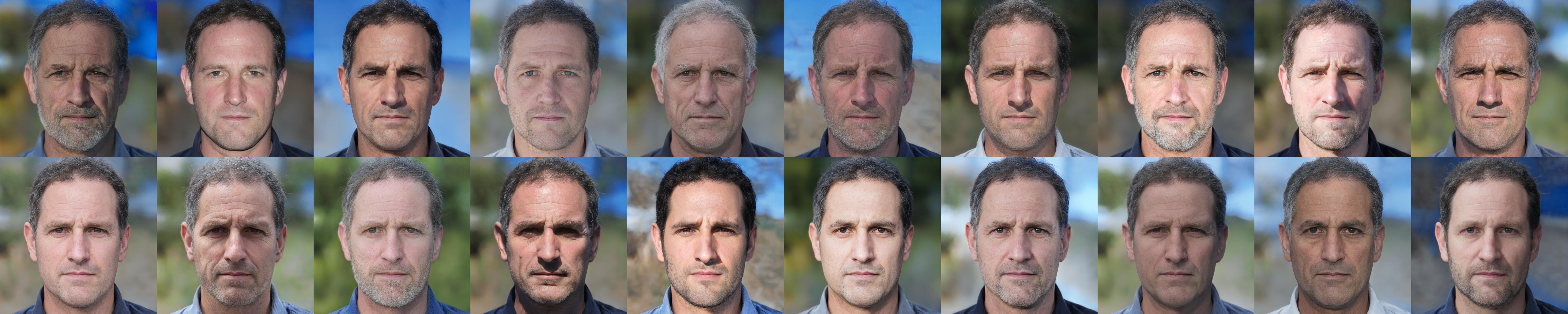}

 %    \vspace{-1.02cm}\captionof*{figure}{}
\end{minipage}
\rotatebox[origin=lc]{90}{\centering \hspace{-1cm} \textbf{$\alpha+$} \hspace{0.7cm}  \textbf{$\alpha-$}}
\captionof{figure}{\label{fig:comparison} Comparison of the top 10 directions for Ganspace\cite{harkonen2020ganspace}, SeFa\cite{shen2020closed} and our method. The first column shows the original image.}
\vspace{-0.3cm}
\end{figure*}

\subsection{Comparison with Unsupervised Methods}
Next, we compare our results with the state-of-the-art unsupervised methods Ganspace \cite{harkonen2020ganspace} and SeFa \cite{shen2020closed}. Ganspace applies PCA to randomly sampled $w$ vectors of StyleGAN2 and uses the resulting principal components as directions. SeFa uses a closed-form approach where it factorizes the weight matrix and uses the resulting eigenvectors with the highest eigenvalues as directions. We used the official implementations for both methods\footnote{\url{http://github.com/harskish/ganspace}, \url{http://github.com/genforce/sefa}} and obtained the top 10 principal components for Ganspace and the top 10 eigenvectors for SeFa methods using the default parameters. Note that since the directions vary by the choice of layers used in SeFa, we experimented with all options (layers 0-1, 2-5, 6-13, all) and chose layers 6-13 because they have the most diverse and semantically meaningful directions (see Appendix \ref{app:sefa}). As can be seen from Figure \ref{fig:comparison}, our method yields more disentangled and diverse directions compared to Ganspace and SeFa. For example, while both Ganspace and SeFA change semantics in the input, such as \textit{gender, age, eyeglasses}, while also changing other semantics  such as \textit{background, position, highlight} at the same time. In contrast, our method performs disentangled edits by changing one semantic at a time. To verify our observations, we also conduct a user study with $N=25$ participants. For the user study, we list the top 10 manipulations of each method along with the original image and ask the following questions:\\

\noindent \textbf{(Q1)} \textit{`How disentangled do you think the change in each image is?'}  \textit{(Note that disentanglement is the degree to which each latent dimension captures at most one attribute.)} (1=Not Disentangled 5=Very disentangled)\\

\noindent  \textbf{(Q2)} \textit{`How semantically meaningful do you think the change in each image is?'} (1=Not Semantically Meaningful 5=Very Semantically Meaningful)

\begin{table}[ht]
\centering 
 \centering
 \begin{tabular}{|p{1.1cm}|c|c|c|}
\hline
\textbf{Model} & \textbf{Ganspace} & \textbf{SeFa} & \textbf{Ours}  \\
\hline\hline
Q1 & 2.46 \scriptsize{$\pm$ 0.45} & 2.91 \scriptsize{$\pm$ 0.41} & 4.32 \scriptsize{$\pm$ 0.31} \\ \hline
Q2 & 3.45 \scriptsize{$\pm$ 0.41} & 3.26 \scriptsize{$\pm$ 0.28} & 4.20 \scriptsize{$\pm$ 0.29}\\
\hline
\end{tabular}
\caption{Comparison with Ganspace and SeFa for \textit{Disentanglement}  $\uparrow$ (Q1) and \textit{Semantically Meaningful} $\uparrow$ (Q2) questions.}
\label{table:unsup_comparison}
\vspace{-0.3cm}
\end{table}
\begin{figure}[t!]\centering

\begin{minipage}{0.7\columnwidth}
    \rotatebox[origin=lc]{90}{\centering \hspace{-0.5cm} \textbf{Ours\phantom{[}}}
    % \hspace{0.003cm}
    \begin{minipage}{0.95\columnwidth}
         \centering
         \includegraphics[width=\textwidth]{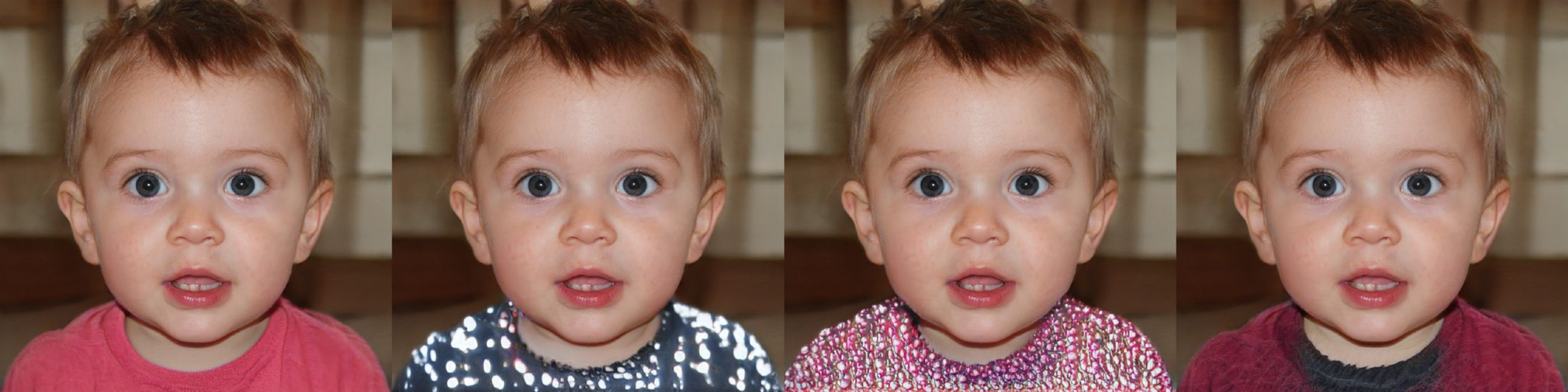}
         \vspace{-1.02cm}\captionof*{figure}{ }
    \end{minipage}
    
    \rotatebox[origin=lc]{90}{\centering \hspace{-0.3cm} \cite{wu2020stylespace}}
    \begin{minipage}{0.95\columnwidth}
         \centering
         \includegraphics[width=\textwidth]{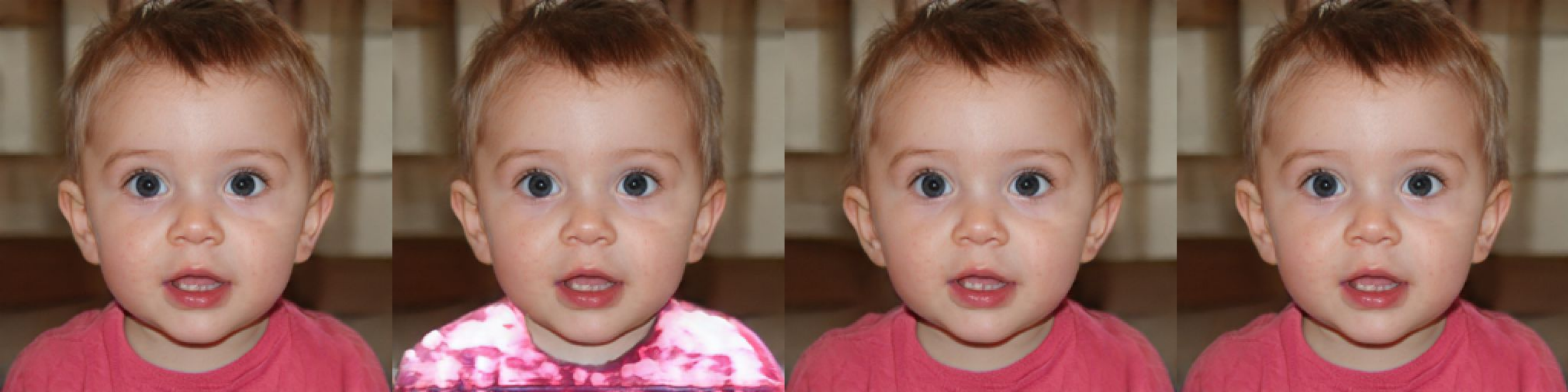}
         \vspace{-1.02cm}\captionof*{figure}{ }
    \end{minipage}
\end{minipage}
\rotatebox[origin=lc]{270}{\centering \hspace{-1.2cm} \textbf{Cloth Region}}

\begin{minipage}{0.7\columnwidth}
    \rotatebox[origin=lc]{90}{\centering \hspace{-0.5cm} \textbf{Ours\phantom{[}}}
    \begin{minipage}{0.95\columnwidth}
         \centering
         \includegraphics[width=\textwidth]{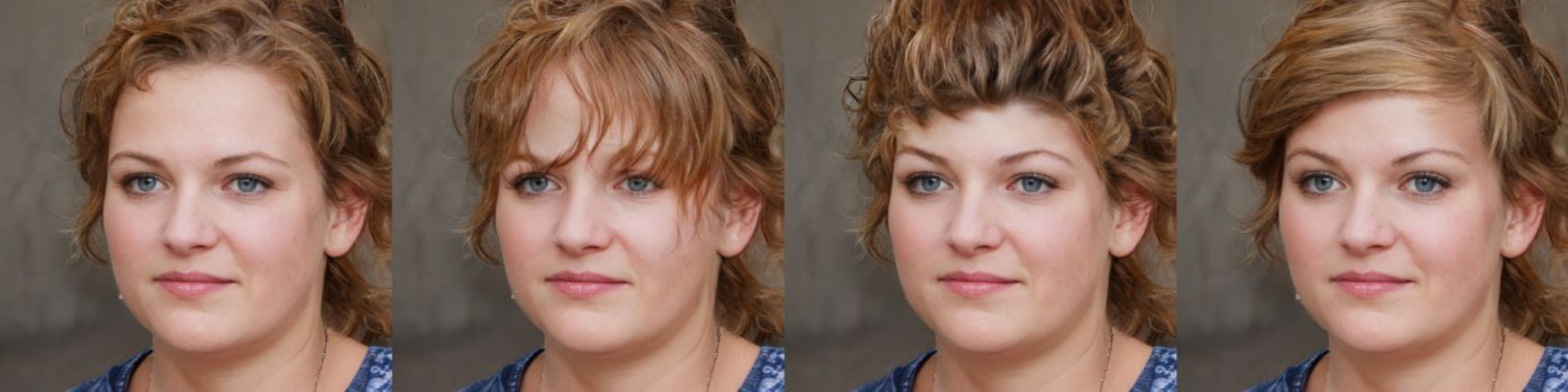}
         \vspace{-1.02cm}\captionof*{figure}{ }
    \end{minipage}
    
    \rotatebox[origin=lc]{90}{\centering \hspace{-0.3cm} \cite{wu2020stylespace}}
    \begin{minipage}{0.95\columnwidth}
         \centering
         \includegraphics[width=\textwidth]{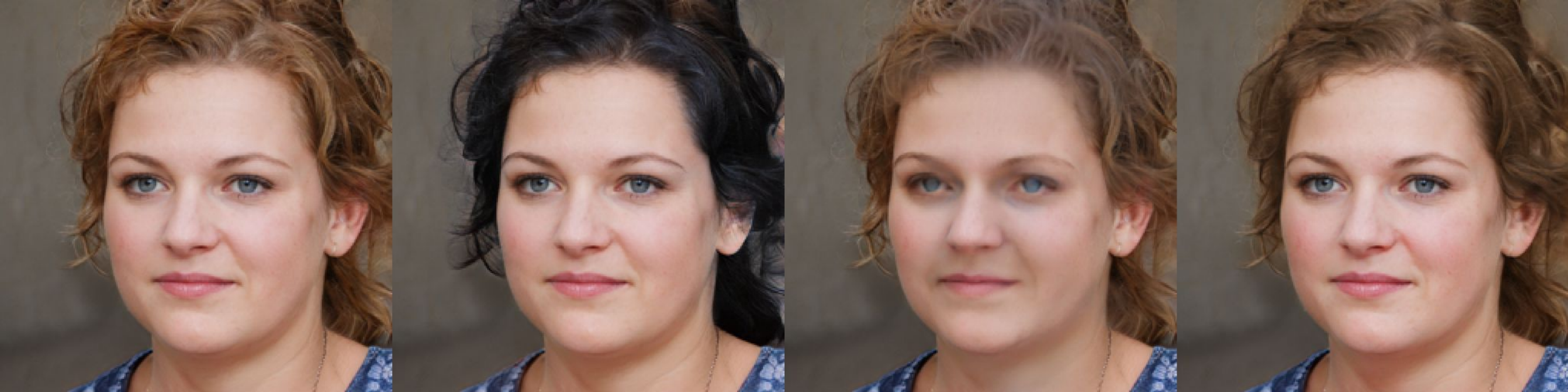}
         \vspace{-1.02cm}\captionof*{figure}{ }
    \end{minipage}
\end{minipage}
\rotatebox[origin=lc]{270}{\centering \hspace{-1.1cm} \textbf{Hair Region}}

\begin{minipage}{0.7\columnwidth}
    \rotatebox[origin=lc]{90}{\centering \hspace{-0.5cm} \textbf{Ours\phantom{[}}}
    \begin{minipage}{0.95\columnwidth}
         \centering
         \includegraphics[width=\textwidth]{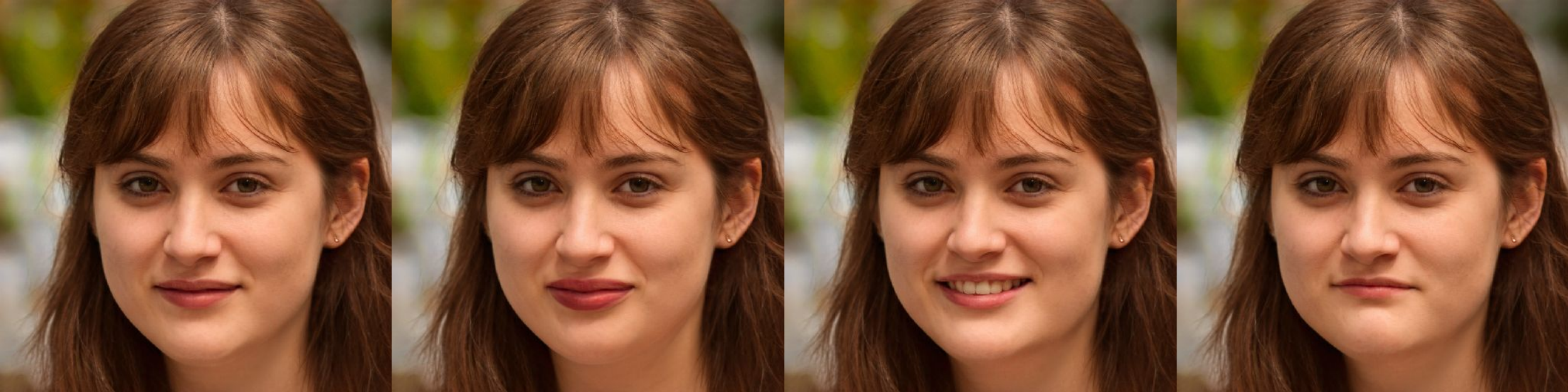}
         \vspace{-1.02cm}\captionof*{figure}{ }
    \end{minipage}
    
    \rotatebox[origin=lc]{90}{\centering \hspace{-0.3cm} \cite{wu2020stylespace}}
    \begin{minipage}{0.95\columnwidth}
         \centering
         \includegraphics[width=\textwidth]{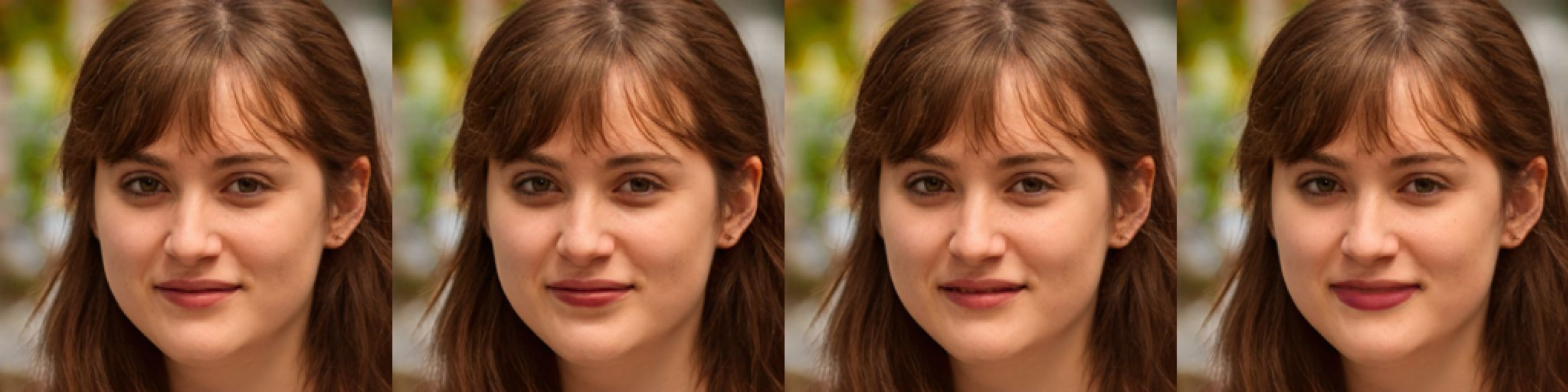}
         \vspace{-1.02cm}\captionof*{figure}{ }
    \end{minipage}
\end{minipage}
\rotatebox[origin=lc]{270}{\centering \hspace{-1.2cm} \textbf{Mouth Region}}

\vspace{-0.2cm}

\captionof{figure}{\textbf{Comparison of channels retrieved using our method and \cite{wu2020stylespace}} for \textit{cloth, hair} and \textit{mouth} regions. Our method is able to capture more diverse channels for a given region.}
\label{fig:stylespace_comparison}
\vspace{-0.5cm}
\end{figure}

As can be seen from Table \ref{table:unsup_comparison}, our method has more disentangled and semantically meaningful directions. All results are statistically significant with a \emph{p}-value of $< 0.0001$. Our method showed a significant performance especially on the \textit{disentanglement} question, with an improvement of $49\%$ over the closest competitor since we operate in $\mathcal{S}$-space, while other methods operate in $\mathcal{W}$-space.

\subsection {Comparison with Supervised Methods}
Both our work and \cite{wu2020stylespace} use stylespace to find style channels that can be used as directions. While our method proposes an unsupervised method for finding the top channels in stylespace, \cite{wu2020stylespace} uses a supervised approach where channels are retrieved based on a specific region (such as \textit{mouth}) or based on a specific attribute classifier. Since \cite{wu2020stylespace} does not provide a way to list the top channels in stylespace, we compare our results with \cite{wu2020stylespace} as follows: we select three regions; \textit{hair, mouth} and \textit{background}. Then, using the official implementation of \cite{wu2020stylespace}\footnote{\url{http://github.com/betterze/StyleSpace}}, we determined top 3 channels for a given region. For our method, we determined  3 clusters with the highest match for a given region and selected a random channel from the obtained clusters. Figure \ref{fig:stylespace_comparison} shows the results for both methods. As can be seen from the figure, our method is able to obtain diverse channels for the regions \textit{clothing, hairstyle} and \textit{mouth}. To verify our observations, we also conduct a user study with $N=25$ participants. We list the results for each method with the original image on the left and ask the question \textit{`How diverse do you think the changes are? (1=Not Diverse 5=Very Diverse)'} to participants \footnote{Note that since both methods use the $\mathcal{S}$-space for disentangled edits, we do not compare for disentanglement.}. As can be seen from the results in Table \ref{table:sup_comparison}, our method showed significantly better diversity than \cite{wu2020stylespace} with a \emph{p}-value of $< 0.0001$. This is due to the fact that \cite{wu2020stylespace} retrieves channels without considering their similarity, while our method considers channels from different clusters.

\begin{table}[ht]
\centering 
 \begin{tabular}{|p{1.1cm}|c|c|}
\hline
\textbf{Model} & \textbf{\cite{wu2020stylespace}} & \textbf{Ours} \\
\hline\hline
Cloth & 2.26 \scriptsize{$\pm$ 1.63} & 4.32 \scriptsize{$\pm$ 0.48}\\
Hair & 2.68 \scriptsize{$\pm$ 1.16} & 4.35 \scriptsize{$\pm$ 0.18} \\
Mouth & 2.16 \scriptsize{$\pm$ 0.38} & 3.64 \scriptsize{$\pm$ 0.64} \\
\hline
\end{tabular}
\caption{Comparison with supervised method \cite{wu2020stylespace} on \textit{Diversity} $\uparrow$.}
\label{table:sup_comparison}
\vspace{-0.5cm}
\end{table} 

\subsection{Applications}
Our framework also opens up possibilities for interesting applications that help users discover new directions.

\textbf{Interactive Editing} Users can navigate the stylespace by drawing  a region of interest such as \textit{hair} and retrieving relevant clusters and corresponding channels. Figure \ref{fig:channel-filter} shows the \textit{background} region with the retrieved clusters (a random channel from each cluster is shown). \textit{See Appendix \ref{app:interactive} for more examples}.

\begin{figure}[ht!]\centering

\rotatebox[origin=lc]{90}{\centering \hspace{-1.2cm} \scriptsize{\textbf{Background Attribute}}}
\begin{minipage}{.6\columnwidth}
     \centering
     \includegraphics[width=\columnwidth]{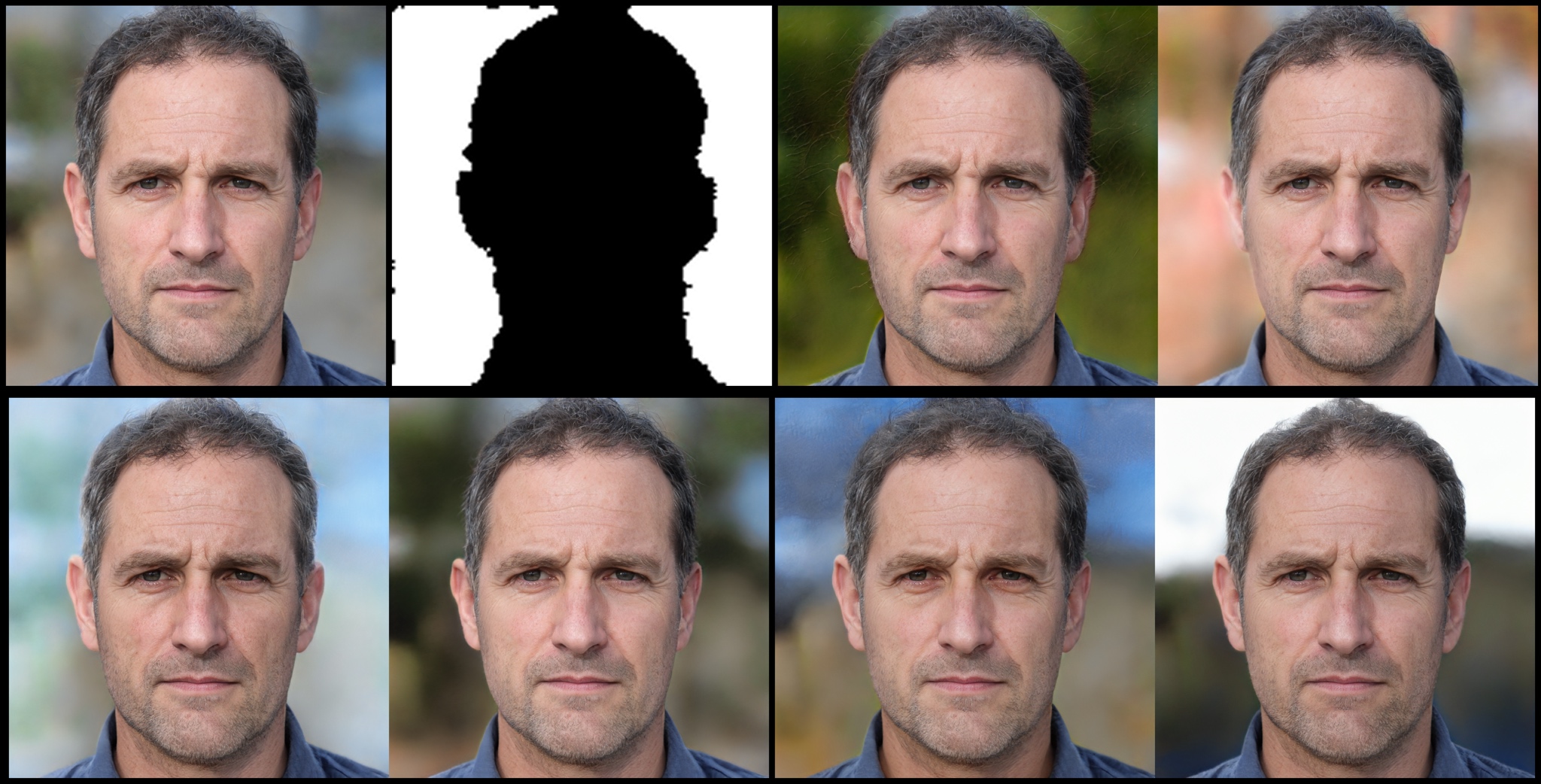}
     \vspace{-1.02cm}\captionof*{figure}{ }
\end{minipage}
    
\captionof{figure}{\label{fig:channel-filter} \textbf{Filtered clusters based on a region specified by the user}. The two images in the upper left show the input image and the user-specified region. The remaining images show randomly selected channels from the retreived clusters.}
\vspace{-0.3cm}
\end{figure}

\textbf{Exploration Platform}
We also provide a web-based platform called \textit{Style Atlas} at {\small \url{http://catlab-team.github.io/styleatlas}} where users can explore the stylespace in a fine-grained way (see Appendix \ref{app:atlas} for a view of the platform). This tool allows users to explore the manipulations made by specific channels based on the region and discover style channels of interest. 

\section{Social Impact and Limitations}
\label{sec:limitations}
Our method uses a pre-trained GAN model as input, so it is limited to manipulating GAN -generated images. However, it can be extended to real images using GAN inversion methods \cite{zhu2020domain} by encoding the real images into the latent space. Like any image synthesis tool, our method poses similar misuse concerns and dangers, as it can be applied to images of people or faces for malicious purposes, as discussed in \cite{korshunov2018deepfakes}. Our method currently applicable to style-based GAN methods such as StyleGAN2, since it directly benefits from the stylespace. We also note that our edits mainly depend on a single StyleGAN2 channel to make disjoint changes. However, it has been shown that more complex edits like \textit{aging} depend on multiple channels, such as \textit{white hair, wrinkles, or glasses} \cite{wu2020stylespace}. We leave the extension to find combinations of style channels and the exploration of our framework to other GAN models such as BigGAN to future work.
 
\section{Conclusion}
\label{sec:conclusion}
In this work, we consider the selection of diverse directions in the latent space of StyleGAN2 as a coverage problem. We formulate our framework as a submodular optimization  for which we provide an efficient solution. Moreover, we provide a complete guide to the stylespace in which one can explore hundreds of diverse directions formed by style channels using clusters. In our experiments, we have shown that our method can identify a variety of manipulations, and performs diverse and disentangled edits.

\textbf{Acknowledgements}
This publication has been produced benefiting from the 2232 International Fellowship for Outstanding Researchers Program of TUBITAK (Project No:118c321). 

%%%%%%%%% REFERENCES
{\small
\bibliographystyle{ieee_fullname}
\bibliography{egbib}
}

\appendix

\begin{figure*}
\centering
        \includegraphics[width=2\columnwidth]{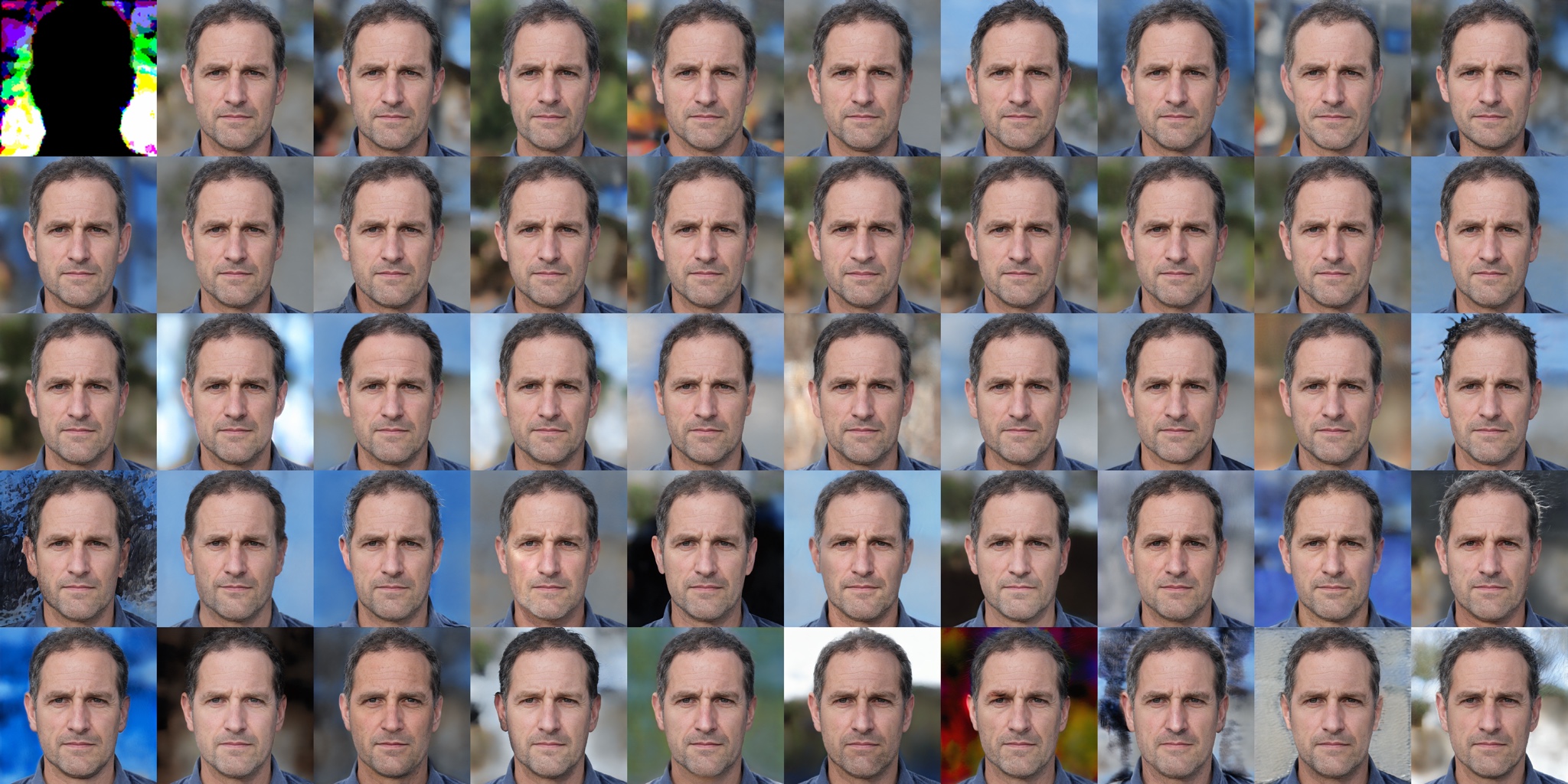}
        %\vspace{-0.5cm}

\caption{\label{fig:background}Various background channels in the style space retrieved based on the heatmap on the upper left.}
%\vspace{-0.7cm}
\end{figure*}
\section{Background channels}
\label{app:bg}

Background is one of the most popular types of edits offered in stylespace. Figure \ref{fig:background} shows various  background channels in stylespace retrieved.

\begin{figure}\centering

\rotatebox[origin=lc]{90}{\centering \hspace{-1.1cm} \textbf{Hair Attribute}}
\begin{minipage}{.9\columnwidth}
     \centering
     \includegraphics[width=\columnwidth]{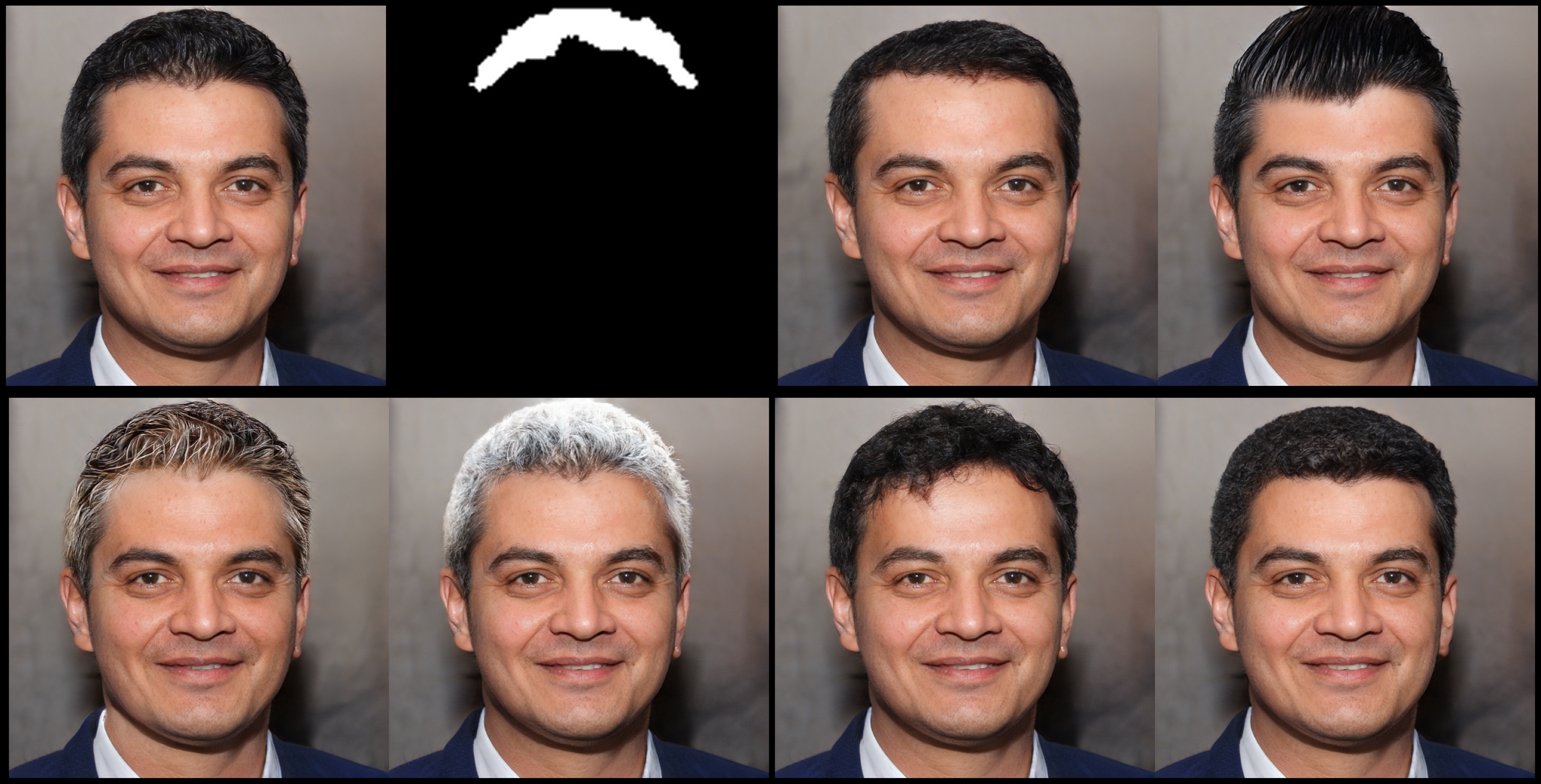}
     \vspace{-1.02cm}\captionof*{figure}{ }
\end{minipage}

\rotatebox[origin=lc]{90}{\centering \hspace{-1.25cm} \textbf{Mouth Attribute}}
\begin{minipage}{.9\columnwidth}
     \centering
     \includegraphics[width=\columnwidth]{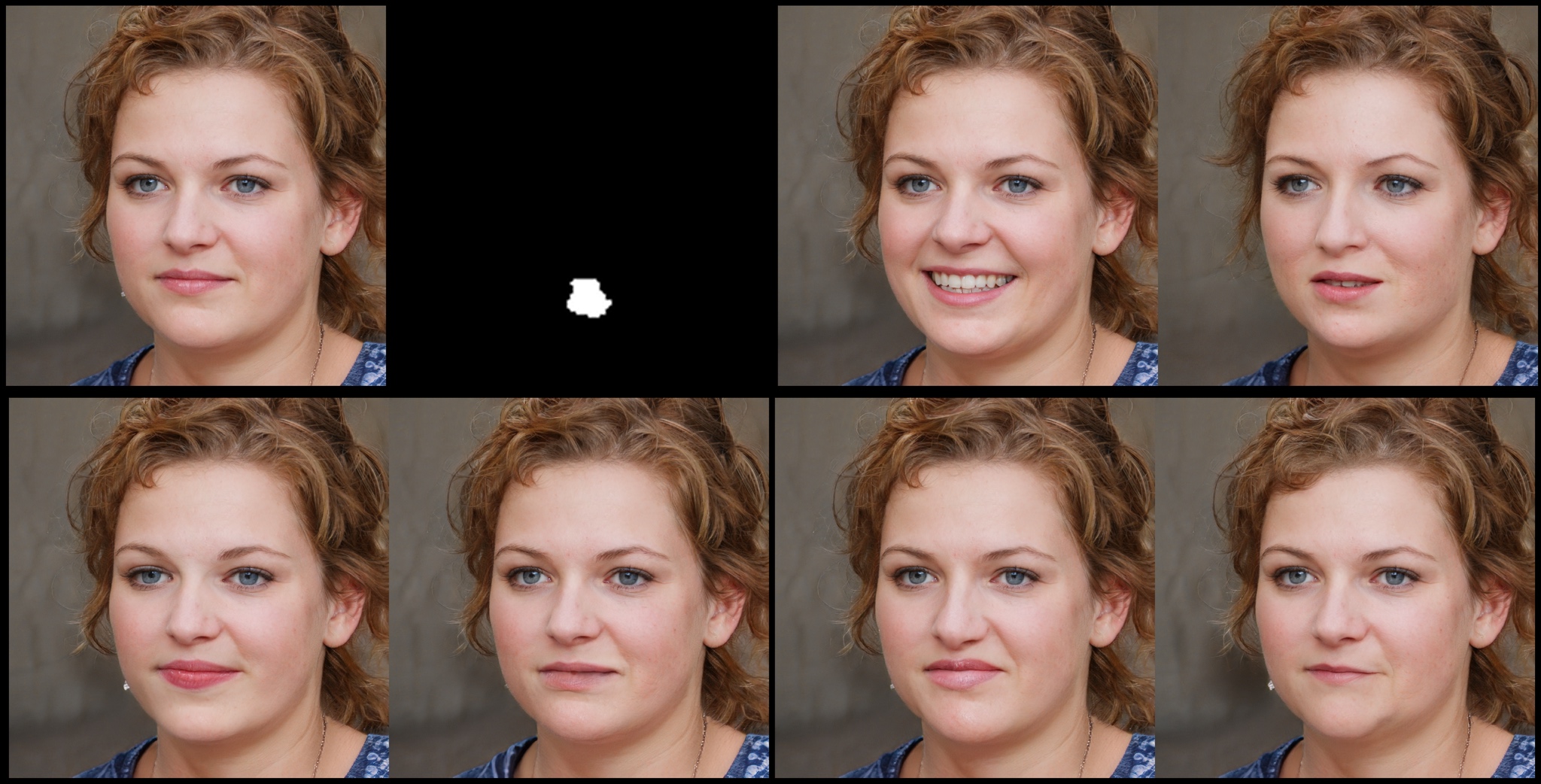}
     \vspace{-1.02cm}\captionof*{figure}{ }
\end{minipage}

\captionof{figure}{\label{fig:channel-filter_appendix} \textbf{Filtered clusters based on a region specified by the user}. The two images in the upper left show the input image and the region specified by the user, while the other images show a randomly selected channel from each retrieved cluster.}
\end{figure}

\begin{figure}
\centering
        \includegraphics[width=0.8\columnwidth]{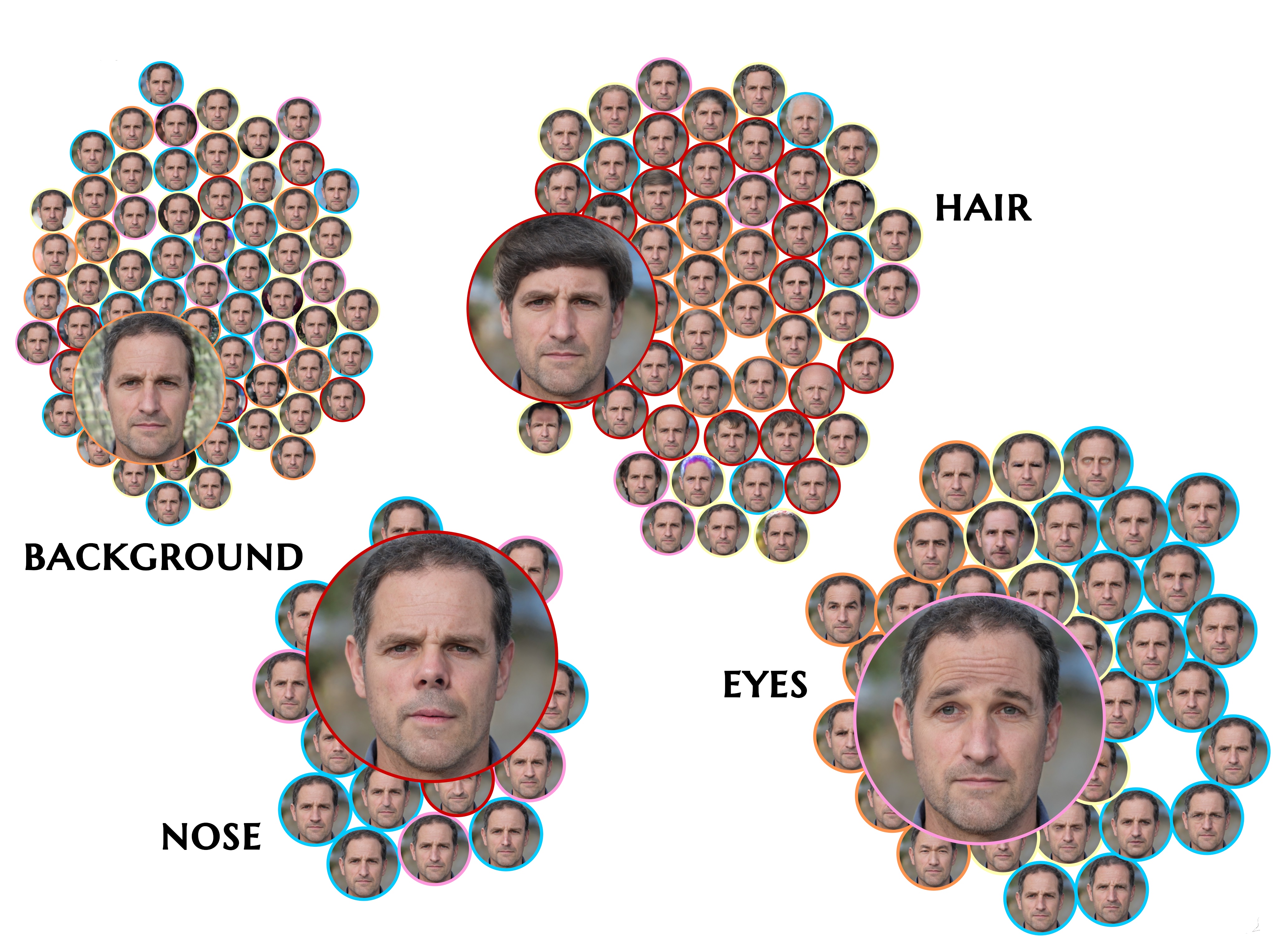}
        %\vspace{-0.5cm}

\caption{A view of the stylespace exploration platform where each group represents a different region, such as \textit{nose} or \textit{eyes}. The bubbles represent manipulation done by a particular channel. The colors around the bubbles represent different layers (zoom for better view).}
\label{fig:styleatlas}
\vspace{-0.5cm}
\end{figure}

\section{Different Options for SeFa}
\label{app:sefa}
SeFa \cite{shen2020closed} uses different layer options when producing the directions. We chose the option 6-13 for comparison since it offers the most semantically meaningful edits compared to other options. We list other options in Figure \ref{fig:sefa}.

\begin{figure*}
\centering
(a)
\begin{minipage}{0.7\textwidth}
     \centering
     \includegraphics[width=\textwidth]{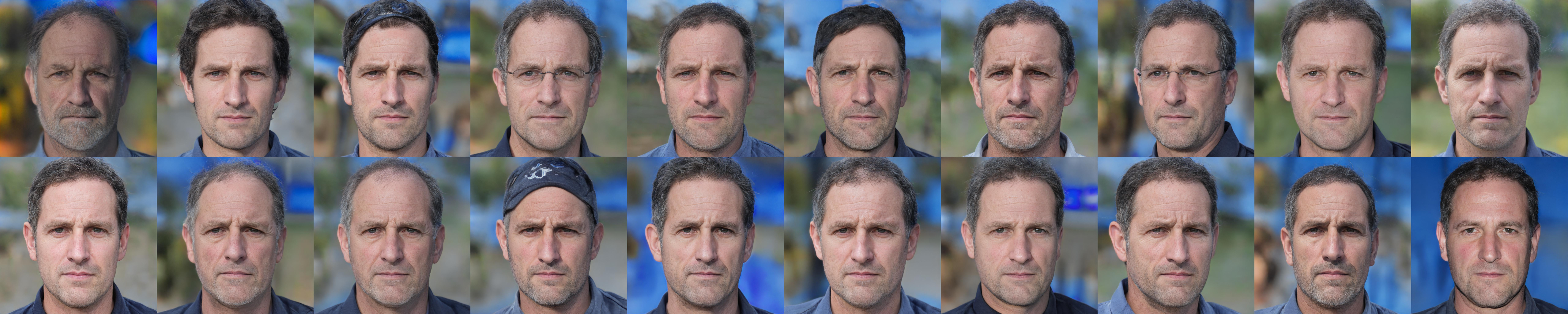}
\end{minipage}

(b)
\begin{minipage}{0.7\textwidth}
     \centering
     \includegraphics[width=\textwidth]{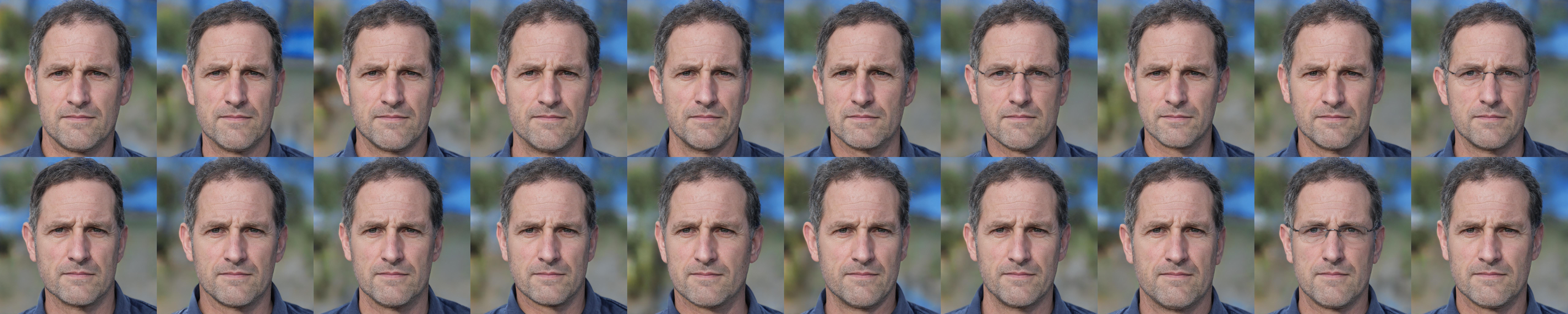}
\end{minipage}

(c)
\begin{minipage}{0.7\textwidth}
     \centering
     \includegraphics[width=\textwidth]{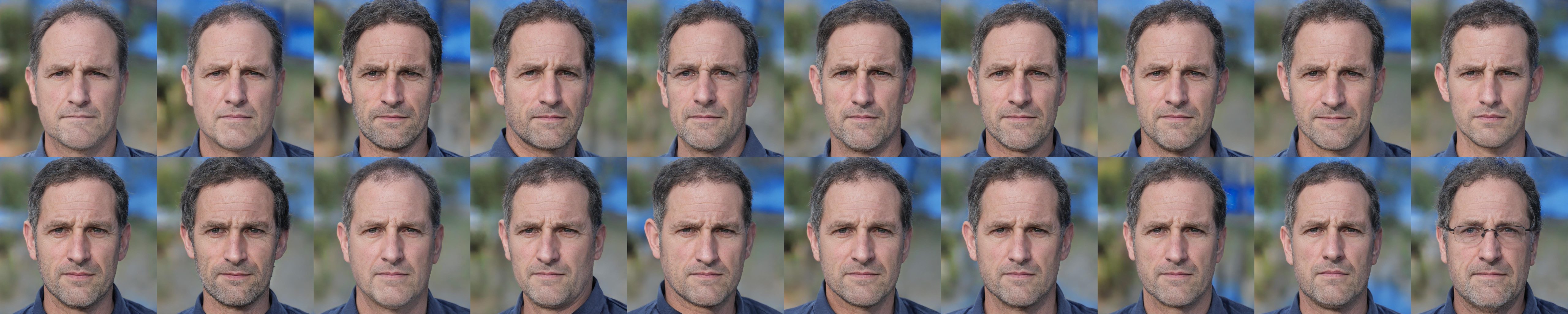}
\end{minipage}
 
 \vspace{-0.2cm}
\captionof{figure}{\label{fig:sefa} Different options for SeFa\cite{shen2020closed}: all layers (a), 0-1 layers (b), and 2-5 layers (c).}
\vspace{-0.2cm}
\end{figure*}

\section{Interactive Editing}
\label{app:interactive}

Additional samples for filtered clusters based on a region specified by the user is shown in Figure \ref{fig:channel-filter_appendix}.

\section{Style Atlas Platform}
\label{app:atlas}
A view of the stylespace exploration platform where each group represents a different region, such as \textit{nose} or \textit{eyes} is shown in Figure \ref{fig:styleatlas}.

\end{document}